\documentclass[prd,11pt]{article}
\pdfoutput=1 

\usepackage{geometry}
\usepackage{authblk}
\usepackage{latexsym}

\usepackage{enumerate}
\usepackage[inline]{enumitem}
\usepackage{amsmath,amssymb}
\usepackage{amsfonts,dsfont}
\usepackage{nicefrac}
\usepackage{microtype}
\usepackage{mathtools}

\usepackage{sidecap}
\usepackage{caption}
\usepackage{subcaption}
\usepackage{wrapfig}
\usepackage{enumitem}
\usepackage{algorithm}
\usepackage[normalem]{ulem}
\usepackage{amssymb}
\usepackage{multicol}
\usepackage{adjustbox}
\usepackage{multirow}
\usepackage{color}
\usepackage{xspace}
\usepackage{CJKutf8}

\PassOptionsToPackage{numbers}{natbib}
\usepackage{natbib}

\usepackage[colorinlistoftodo   s,prependcaption,textsize=tiny]{todonotes}

\usepackage{user_definitions}
\usepackage{xspace}
\usepackage{algorithm}
\usepackage{algorithmicx}
\usepackage[no end]{algpseudocode} 
\usepackage{amsmath,amssymb}
\usepackage{mathtools}
\usepackage{sidecap}
\usepackage{booktabs}
\algdef{SE}[DOWHILE]{Do}{doWhile}{\algorithmicdo}[1]{\algorithmicwhile\ #1}%

\newcommand{\dataset}{X}

\newcommand{\tree}{{H}}

\newcommand{\powerset}{\mathbb{P}}

\newcommand{\trellis}{\mathbb{T}} %\Tcal
\newcommand{\trellisVertex}{\mathbb{V}}

\newcommand{\dasguptaGraph}{\mathcal{W}}

\newcommand{\astar}{A*\xspace}

\newcommand{\frontier}{\ensuremath{\Pi}}
\newcommand{\maxsplit}{\ensuremath{\Phi}}
\newcommand{\argmaxsplit}{\ensuremath{\Xi}}

\newcommand{\lchild}{\ensuremath{\ell}}
\newcommand{\rchild}{\ensuremath{r}}
\newcommand{\children}{\mathbb{C}}

\newcommand{\leaves}{\textsf{lvs}}

\newcommand{\EfunS}{\ensuremath{\psi}}
\newcommand{\EfunT}{\ensuremath{\phi}}

\newcommand{\ch}{\ensuremath{\textsc{ch}}}

\newcommand{\Hfun}{\mathbb{H}}

\newcommand{\alltrees}{\ensuremath{\mathbb{H}}}
\newcommand{\partialtree}{\ensuremath{\Breve{H}}}

\newcommand{\hone}{$\textsf{h1}(\cdot)$\xspace}
\newcommand{\hzero}{$\textsf{h0}(\cdot)$\xspace}

\def\beq{\begin{equation}}
\def\eeq{\end{equation}}
\newcommand{\bea}{\begin{eqnarray}\begin{aligned}}
\newcommand{\eea}{\end{aligned}\end{eqnarray}}

\usepackage[american]{babel}
\usepackage{natbib} % has a nice set of citation styles and commands
\usepackage{mathtools} % amsmath with fixes and additions
\usepackage{booktabs} % commands to create good-looking tables
\usepackage{tikz} % nice language for creating drawings and diagrams

\title{Exact and Approximate Hierarchical Clustering Using A*}

\begin{document}
\vspace{1cm}
\author[1]{Craig S. Greenberg\thanks{The three authors contributed equally to this work.}$^,$}
\author[2]{Sebastian Macaluso$^{\ast, }$}
\author[3]{Nicholas Monath$^{\ast, }$}
\author[4]{\\ Avinava Dubey}
\author[5]{Patrick  Flaherty}
\author[4]{Manzil Zaheer}
\author[4]{Amr Ahmed} 
\author[2]{\\ Kyle Cranmer}
\author[3]{Andrew McCallum}

\date{} 
\affil[1]{\small{National Institute of Standards and Technology, USA}}
\affil[2]{\small{Center for Cosmology and Particle Physics \& Center for Data Science, New York University, USA}}
\affil[3]{\small{College of Information and Computer Sciences, University of Massachusetts Amherst, USA}}
\affil[4]{\small{Google Research, Mountain View, CA, USA}}
\affil[5]{\small{Department of Mathematics and Statistics, University of Massachusetts Amherst, USA}}
\providecommand{\keywords}[1]{\textit{Keywords:} #1}

\maketitle

\begin{abstract}
Hierarchical clustering is a critical task in numerous domains. 
Many approaches are based on heuristics and the properties of the resulting clusterings are studied post hoc. 
However, in several applications, there is a natural cost function that can be used to characterize the quality of the clustering. 
In those cases, hierarchical clustering can be seen as a combinatorial optimization problem.
To that end, we introduce a new approach based on \astar search. We overcome the prohibitively large search space by combining \astar with  a novel \emph{trellis} data structure. 
This combination results in an exact algorithm that scales  beyond previous state of the art, from a search space with $10^{12}$ trees to $10^{15}$ trees, and an approximate algorithm that improves over baselines, even in enormous search spaces that contain more than $10^{1000}$ trees.
We empirically demonstrate that our method achieves substantially higher quality results than  baselines for a particle physics use case and other clustering benchmarks. 
We describe how our method provides significantly improved theoretical bounds on the time and space complexity of A* for clustering. 
\end{abstract}
\section{Introduction}
Hierarchical clustering has been applied in a wide variety of settings such as scientific discovery \cite{sorlie2001gene}, personalization \cite{zhang2014taxonomy}, entity resolution for knowledge-bases \cite{green2012entity,levin2012citation,vashishth2018cesi}, and jet physics~\cite{ Cacciari:2008gp,Catani:1993hr,Dokshitzer:1997in, Ellis:1993tq}. While much work has focused on approximation methods for relatively large datasets \cite{bateni2017affinity,fichtenberger2013bico,garg2006pbirch,monath2019scalable,naumov2020objective,zhang1997birch, dubey2014,hu2015large,monath2020scalable,dubey2020distributed,monath2021dag}, there are also important use cases of hierarchical clustering that demand
exact or high-quality approximations \cite{greenberg2020compact}.  
This paper focuses on one such application of hierarchical clustering in jet physics: inferring \emph{jet structures}.

\textbf{Jet Physics:}
Particle collisions in collider experiments, such as the Large Hadron Collider (LHC) at CERN, produce new particles that go through a series of successive binary splittings, termed a {\it showering process}, which finally leads to a {\it jet}: a collimated set of stable-charged and neutral particles that are measured by a detector. 
A schematic representation of this process is shown in Figure \ref{fig:particleexample}, where the jet constituents are the (observed) leaf particles in solid dark blue  and the intermediate (latent) particles can be identified as internal nodes of the hierarchical clustering.
\begin{figure}
    \centering
    \includegraphics[width=0.55\textwidth]{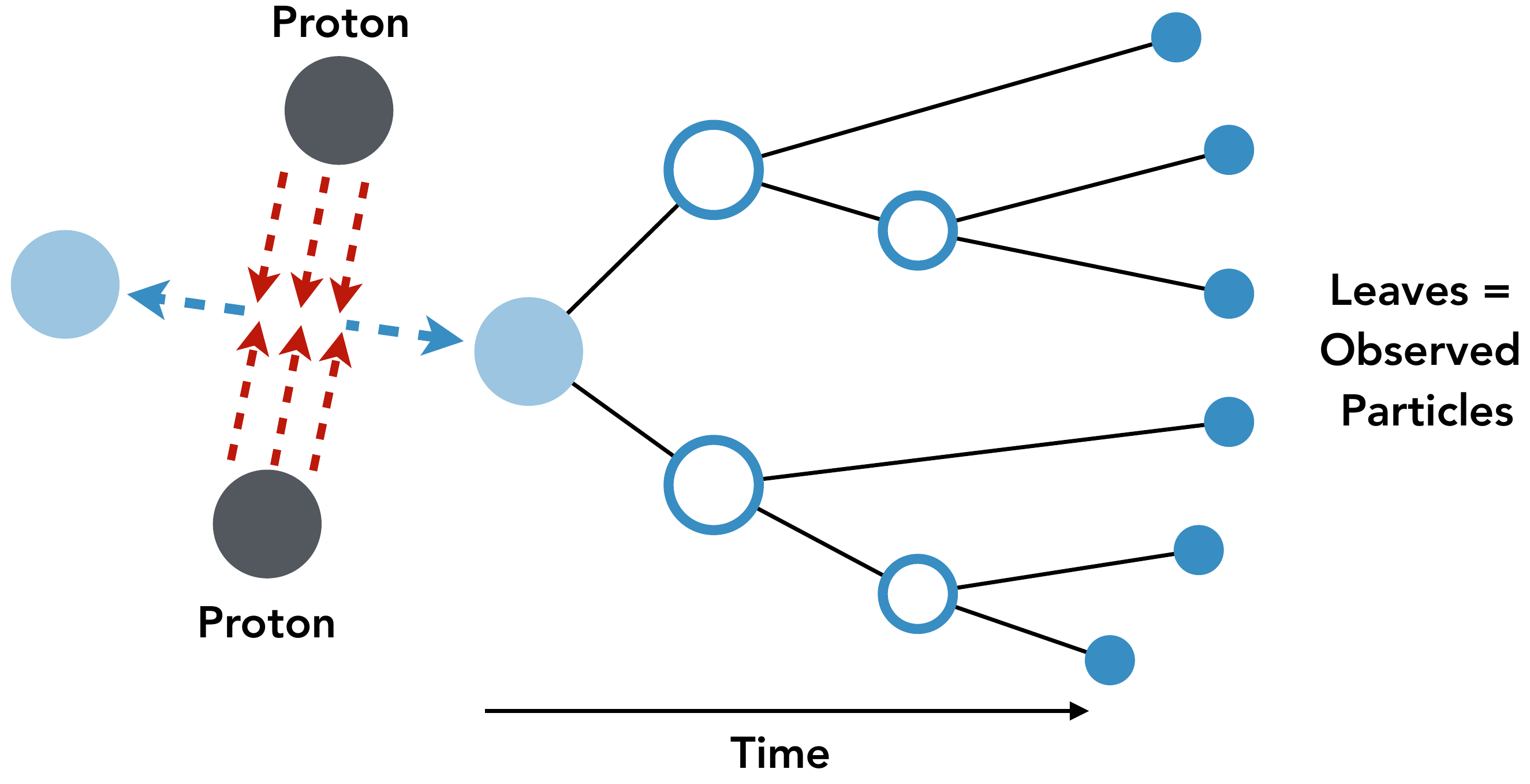}
    \caption{\small{\textbf{Jets as hierarchical clustering}. Schematic representation of a process producing a jet at the Large Hadron Collider at CERN. Constituents of the incoming protons interact and produce two new particles (solid light blue). Each of them goes through successive binary splittings until giving rise to stable final state particles (solid dark blue), which hit the detector. These final state particles (leaves of a binary tree) form a jet, and each possible hierarchical clustering represents a different splitting history. } 
    \label{fig:particleexample}}
\end{figure}
There has been a significant amount of research into jet clustering over the decades, i.e. reconstructing the showering process (hierarchical clustering) from the observed jet constituents (leaves). The performance of these algorithms is often the bottleneck to precision measurements of the Standard Model of particle physics and searches for new physics. 
The LHC generates enormous datasets, but for each recorded collision there is a hierarchical clustering task  with typically 10 to 100 particles (leaves). Despite the relatively small number of elements in these applications, exhaustive solutions are intractable, and current exact methods, e.g., \cite{greenberg2020compact}, have limited scalability.
The industry standard uses agglomerative clustering techniques, which are greedy \cite{Cacciari:2011ma}. 
Thus, they tend to find low-likelihood hierarchical clusterings, which results in a poor inference of the properties of the initial state particles (solid light blue circles in Figure \ref{fig:particleexample}). As a result, their type could be misidentified (e.g., top quark, $W$ boson, gluon, etc.). 
Also, the success of deep learning jet classifiers \cite{Kasieczka:2019dbj} (for these initial state particles) that do not use clustering information is evidence for the deficiencies of current clustering algorithms since the {\it right} clustering would mean that traditional clustering-based physics observables would be essentially optimal in terms of classification performance.
Thus, exact and high-quality approximations in this context would be highly beneficial for data analyses in experimental particle physics.

Traditional algorithms, such as greedy approximation methods can 
converge to local optima,
and have little ability to re-consider possible alternatives encountered at previous steps in their search for the optimal clustering. This naturally raises the question of how well-established, \emph{non-greedy} search algorithms, such as \astar, can be applied to hierarchical clustering.

An immediate challenge in using \astar to search for the optimal tree structure for a given objective is the vast size of the space of hierarchical clusterings.
There are \emph{many} possible hierarchical clusterings for $n$ elements, specifically $(2n-3)!!$, where $!!$ indicates double factorial. 
A na\"ive application of \astar would require super-exponential time and space.
Indeed, the only previous work exploring the use of \astar for clustering, \cite{daume2007fast},  uses \astar to find MAP solutions to Dirichlet Process Mixture Models and has no bounds on the time and space complexity of the algorithm.

\subsection{Theoretical contributions of this paper}
\begin{itemize}
    \item \textbf{\astar Datastructure}. Inspired by \cite{greenberg2018compact,greenberg2020compact}, we present a data structure to compactly encode the state space and objective value of hierarchical clusterings for search, using a sequence of nested min-heaps (\S \ref{sec:astar_on_trellis}). 
    \item \textbf{Time \& Space Bounds}.  Using this structure, we are able to bound the running time and space complexity of using \astar to search for clusterings, an improvement over previous work (\S \ref{sec:astar_on_trellis}). 
\end{itemize}

\subsection{Empirical contributions of this paper}
\begin{itemize}
    \item \textbf{Jet Physics}. We also demonstrate empirically that \astar can 
find exact solutions to larger jet physics datasets 
than previous work \cite{greenberg2020compact}, and can improve over benchmarks among approximate methods in data sets with enormous search spaces (exceeding $10^{300}$ trees). (\S \ref{sec:jet_experiments}). 
    \item \textbf{Clustering benchmarks}. We find that \astar search finds solutions with improved hierarchical correlation clustering cost compared with common greedy methods on benchmark datasets \cite{kobren2017hierarchical}.(\S \ref{sec:benchmarks}). 
\end{itemize}

\section{Preliminaries}
\label{sec:prelim}

A hierarchical clustering is a recursive partitioning of a dataset into nested subsets:

\begin{definition}{\textbf{\emph{(Hierarchical Clustering)}}}
  \label{defn:hclustering}
  Given a dataset of elements, $\dataset = \{x_i\}_{i=1}^N$, a
  \textbf{hierarchical clustering}, $\tree$, is a set of nested subsets of $\dataset$, s.t. $\dataset \in \tree$, $\{\{x_i\}\}_{i=1}^N \subset \tree $, and $\forall \dataset_i, \dataset_j \in \tree$, either $\dataset_i \subset \dataset_j$,  $\dataset_j \subset \dataset_i$, or $\dataset_i \bigcap \dataset_j = \emptyset$. Further, $\forall \dataset_i \in \tree$, if  $\exists \dataset_j \in \tree$ s.t. $\dataset_j \subset \dataset_i$, then $\exists \dataset_k \in \tree$ s.t. $\dataset_j \bigcup \dataset_k = \dataset_i$.
\end{definition}
Consistent with our application in jet physics and previous work \cite{greenberg2020compact}, we limit our consideration to hierarchical clusterings with binary branching factor.\footnote{Binary trees encode at least as many \emph{tree consistent partitions} as multifurcating / n-ary trees \cite{blundell2010bayesian,greenberg2020compact} and for well known cost functions, such as Dasgupta's cost, provide trees of lower cost than multifurcating / n-ary trees \cite{dasgupta2016cost}.}

Hierarchical clustering cost functions represent the quality of a hierarchical clustering for a particular dataset.
In this work, we study a general family 
of hierarchical clustering costs defined as the sum over pairs of sibling clusters in the structure (Figure \ref{fig:cost_model}). This family of costs generalizes well known costs like Dasgupta's \cite{dasgupta2016cost}. Formally, the family of hierarchical clustering costs we study in this paper 
can be written as:

\begin{definition}
\label{defn:energy_based_hclustering}
\textbf{\emph{(Hierarchical Clustering Costs)}}
Let $\dataset$ be a dataset, $\tree$ be a hierarchical clustering of $\dataset$, let $\EfunS: 2^{\dataset} \times 2^{\dataset} \rightarrow \RR^{+}$ be a function
describing the cost incurred by a pair of sibling nodes in $\tree$. We define the cost, $\EfunT(\tree)$ of a hierarchical clustering $\tree$ as:
\begin{equation}
    \label{eq:tree-cost}
      \EfunT(\tree) = \sum_{X_L,X_R \in \textsf{sibs}(\tree)} \EfunS(X_L,X_R)
\end{equation}
where $ \textsf{sibs}(\tree) = \{(X_L,X_R)|X_L \in \tree, X_R \in \tree, X_L \cap X_R = \emptyset, X_L \cup X_R \in \tree\}$. 

The goal of hierarchical clustering is then to find the lowest cost $\tree$ among all hierarchical clusterings of $\dataset$, $\alltrees(\dataset)$, i.e., $\argmin_{\tree \in \alltrees(\dataset)} \EfunT(\tree)$ .
\end{definition}
\begin{figure}
    \centering
    \includegraphics[width=0.38\textwidth]{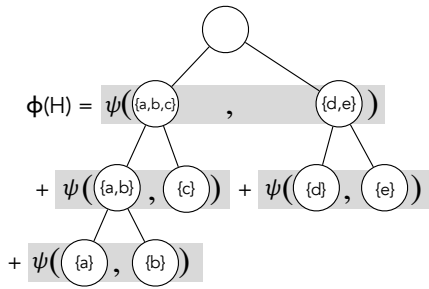}
    \caption{\textbf{Family of Hierarchical Clustering Cost Functions}. This paper studies cost functions that decompose as the sum of a cost term computed for each
    pair of sibling nodes in the tree structure.}
    \label{fig:cost_model}
\end{figure}
% https://www.icloud.com/keynote/0Hh3Dd4OWdU-pPZXmsJmySrUA#a_star_figures 
\vspace{-0.5cm}
\section{A* Search for Clustering}
\label{sec:methods}

The fundamental challenge of applying \astar to clustering
is the massive state space. Na\"ive representations of the \astar state space and frontier require explicitly storing every tree structure, potentially growing to be at least the number of binary trees $(2n-3)!!$.
To overcome this, we propose an approach using a \emph{cluster trellis} to (1) compactly encode states in the space
of hierarchical clusterings (as paths from the root to the leaves of the trellis), and (2) compactly represent the search frontier (as nested priority queues).
We first describe how the cluster trellis represent the A* state space and frontier, and then show how \astar search this data structure to find the lowest cost clustering.

\subsection{\astar Search and State Space}
\astar is a best-first search algorithm for finding the minimum cost path between a starting node and a goal node in a weighted graph. Following canonical notation, the function $f(n) = g(n) + h(n)$ determines the ``best'' node to search next, where $g(n)$ is the cost of the path up from the start to node $n$ and $h(n)$ is a heuristic estimating the cost of traveling from $n$ to a goal. 

When the heuristic $h(\cdot)$ is \emph{admissible} (under estimates cost), \astar is admissible and provides an optimal solution. If $h(\cdot)$ is both admissible and \emph{consistent} or \emph{monotone} (the heuristic cost estimate of reaching a goal from state $x$, $h(x)$, is less than the actual cost of traveling from state $x$ to state $y$ plus the heuristic cost estimate of reaching a goal from state y, $h(y)$), then \astar is \emph{optimally efficient}\footnote{It is worth noting that there is a trade-off between the quality of the heuristic and the number of iterations of A*, with better heuristics potentially resulting in fewer iterations at the cost of taking longer to compute. An extreme example being that a perfect heuristic will require no-backtracking, except in cases of ties.}. 

In order to use \astar to search for clusterings, we need to define the search space, the graph being searched over, and the goal states.

\begin{definition}\emph{{(\textbf{Search State / Partial Hierarchical Clustering})}}
We define each state in the search space to be a \textbf{partial hierarchical clustering}, $\partialtree$, which is a subset of some hierarchical clustering of the dataset $\dataset$, i.e., $\exists \tree \in \alltrees(\dataset), \text{ s.t. } \partialtree \subseteq \tree$. 
A state is a \textbf{goal} state if it is a hierarchical clustering, e.g. $\partialtree \in \alltrees(\dataset)$.
\end{definition}

\subsection{ The Cluster Trellis for Compactly Encoding the Search Space}

It is intractable to search over a graph that na\"ively encodes these states, even when using \astar, as the number of goal states is massively large ($(2N-3)!!$), and therefore the total number of states is even larger. Such a na\"ive search could require super exponential space to maintain the the frontier, i.e., the set of states \astar is currently open to exploring.
We propose an alternative encoding that utilizes a trellis data structure to encode the frontier and state space. 
Our proposed encoding reduces the super-exponential space and time required for \astar to find the exact MAP to exponential in the worst case.

Previous work, \cite{greenberg2020compact},
defines the \emph{cluster trellis for hierarchical clustering} (hereafter denoted \emph{trellis}) as a directed acyclic graph, $\trellis$, where, like a hierarchical clustering, 
the leaves of the trellis, $\leaves(\trellis)$, correspond to data points.
The internal nodes of $\trellis$ correspond to clusters, and 
edges in the trellis are defined between child and parent nodes, where children contain a subset of the points contained by their parents. 

\begin{definition}{(Cluster Trellis for Hierarchical Clustering) \cite{greenberg2020compact}}
  \label{defn:hclustering}
  Given a dataset of elements, $\dataset = \{x_i\}_{i=1}^N$, a trellis $\trellis$ is a directed acyclic graph with vertices corresponding to subsets of $\dataset$, i.e., $\trellis = \{T_1,\dots,T_k\} \subseteq \PP(\dataset)$, where $\PP(\dataset)$ is the powerset of $\dataset$. Edges in the trellis are defined between a child node $C\in\trellis$ and a parent $P\in\trellis$ such that $P \setminus C \in \trellis$ (i.e., $C$ and $P\setminus C$ form a two-partition of $P$). 
\end{definition}

Note that the trellis is a graph that compactly represents tree structures, enabling the use of various graph search algorithms for hierarchical clustering.

\begin{figure}
    \centering
    \includegraphics[width=0.7\columnwidth]{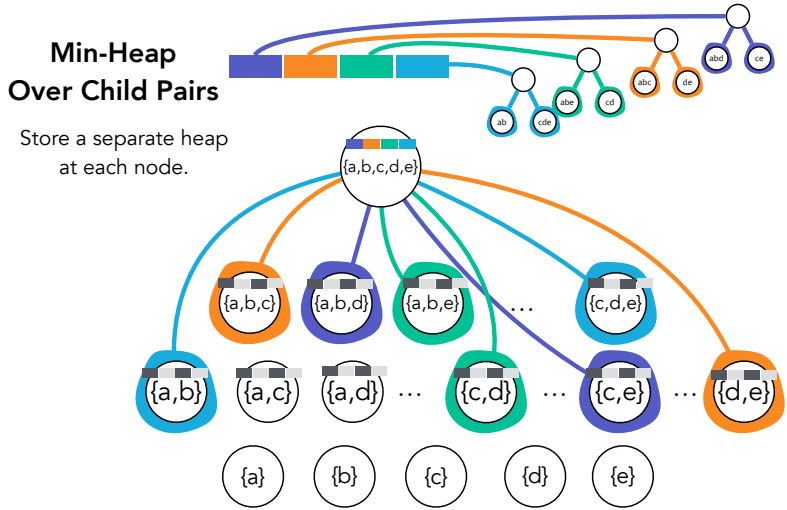}
    \caption{\textbf{Nested Min Heaps} The \astar search space is compactly encoded in the trellis. Each node stores a min-heap ordering pairs of its children. Each pair encodes a two partition (e.g. split) of its parents points. }
    \label{fig:min_heap}
\end{figure}

\subsection{\astar on the Cluster Trellis}
\label{sec:astar_on_trellis}
In this section, we describe how to use of the trellis to run \astar in less than super-exponential time and space.
We show that when using a `full'' trellis, one that contains all possible subsets of the dataset (i.e., $\trellis = \PP(\dataset)$), \astar will find an exact solution, i.e., the optimal (minimum possible cost) hierarchical clustering. 
Interestingly, when the trellis is ``sparse``, i.e., contains only a subset of all possible subsets ($ \trellis \subset \PP(\dataset)$), we further show that \astar finds the lowest cost hierarchical clustering among those encoded by the trellis structure.
In our proceeding analysis, we prove the correctness as well as the time and space complexity of running \astar search on a trellis. 

To run \astar on trellis $\trellis$, we store at each node,
$\dataset_i$, in $\trellis$ a min heap over two-partition pairs, which correspond to the node's children, $\frontier$, which we denote as $\dataset_i[\frontier]$. Each node also stores the value of the lowest cost split at each node, $\dataset[\maxsplit]$, as well as a pointer to children associated with the best split, $\dataset[\argmaxsplit] = \{\dataset_\lchild, \dataset_\rchild\}$.

The starting state of the search refers to the root node corresponding to the full dataset, $\dataset$. 
We compactly encode a partial hierarchical clustering as a set of paths through the trellis structure. 
In particular,
\begin{equation}\label{eq:tree_from_trellis}
\tree_\argmaxsplit(\dataset) = \{\dataset\} \bigcup_{\dataset_c \in \dataset[\argmaxsplit]} \tree_\argmaxsplit(\dataset_c)
\end{equation}
where $\dataset[\argmaxsplit] = \emptyset$ if $\dataset \in \leaves(\tree)$, where $\leaves(\tree) = \{X \ | \ X \in \tree, \ \nexists X' \in \tree, X' \subset X\}$. 
The partial hierarchical clustering can be thought of as starting at the root node, adding to the tree structure the pair of children at the top of the root node's min heap, and then for each child in this pair, selecting the pair of children (grandchildren of the root) from their min-heaps to add to the tree structure. The algorithm proceeds in this way, adding descendants until it reachs nodes that do not yet have min-heaps initialized. 

Each node's min heap, $\dataset_i[\frontier]$, stores five-element tuples:
\begin{align}
    (f_{LR},\ g_{LR},\ h_{LR}, \dataset_L,\ \dataset_R) \in \dataset_i[\frontier]
\end{align}

Exploration in \astar consists of instantiating the min heap 
of a given node, where the heap ordering is determined by $f_{LR} = g_{LR} + h_{LR}$. The value $g_{LR}$ corresponds to the $g$ value of the partial hierarchical clustering rooted at $\dataset_L \cup \dataset_R$, including both the nodes $\dataset_L$ and $\dataset_R$. 
 Formally:
\begin{align}
\label{eq:glr}
    g_{LR} &= \EfunS(\dataset_L, \dataset_R) 
    + \sum_{\dataset_{LL},\dataset_{LR} \in  \textsf{sibs}(\tree_\argmaxsplit(\dataset_L))} \EfunS(\dataset_{LL}, \dataset_{LR})% \nonumber \\
    + \sum_{\dataset_{RL},\dataset_{RR} \in  \textsf{sibs}(\tree_\argmaxsplit(\dataset_R))} \EfunS(\dataset_{RL}, \dataset_{RR})
\end{align}
Recall that $\tree_\argmaxsplit$ represents a complete or partial hierarchical clustering. The value $h$ gives an estimate of the potential 
for all of the leaves of $\tree_\argmaxsplit$, and is defined to be 0 if all of the leaf nodes of $\tree_\argmaxsplit$ are singleton clusters (in which case $\tree_\argmaxsplit$ is a complete hierarchical clustering). Formally:
\begin{align}
\label{eq:hlr}
    h_{LR} = \sum_{\dataset_\ell \in \textsf{lvs}(\tree_\argmaxsplit(\dataset_L \cup \dataset_R))} \Hfun(\dataset_\ell)
\end{align}
where $\Hfun(\dataset_\ell)$ is the objective-specific heuristic function required by A*, and $\Hfun(\dataset_\ell)=0$ if $\dataset_\ell$ is a singleton.

\begin{figure}
\centering
\includegraphics[width=0.7\columnwidth]{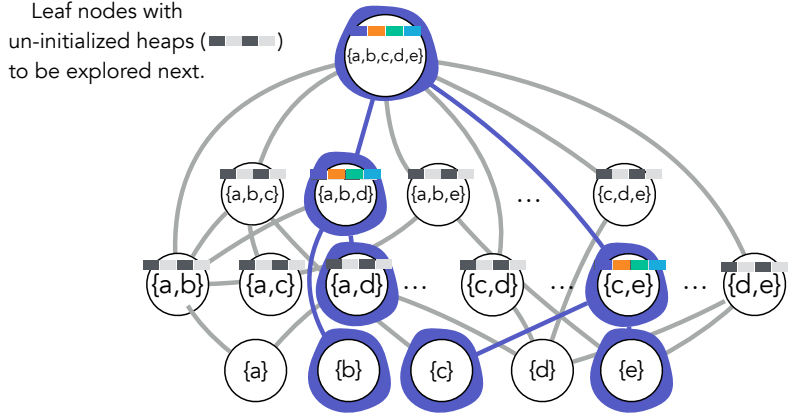}
\caption{\small{\textbf{State Space} A state (i.e., a partial hierarchical clustering) is compactly encoded as paths in the trellis structured, discovered by following the root-to-leaf pointers.
}}
\label{fig:min_heap}
\end{figure}

Given these definitions, \astar can be executed in the following way. 
First, follow the paths from the root of the trellis to find the best partial hierarchical clustering (using Eq \ref{eq:tree_from_trellis}).
If this state is a goal state and the $h_{LR}$ value at the top of the 
min heap is 0, then return this structure. 
If the state is not a goal state, explore the leaves of the partial hierarchical clustering, instantiating
the min heap of each leaf with the children of the leaf nodes and their corresponding $f_{LR}, g_{LR}, h_{LR}$
values (from Eq. \ref{eq:glr} \& \ref{eq:hlr}). Then, after exploration, update 
the heap values of $g_{LR}$ and $h_{LR}$ in each of the ancestor nodes (in the partial 
hierarchical clustering)
of the newly explored leaves. 
See Algorithm \ref{alg:tree-astar} for a summary in pseudocode.

First, we show that Algorithm \ref{alg:tree-astar} will find the exact optimal hierarchical clustering among all those represented in the trellis structure. 

\begin{theorem}{\textbf{\emph{(Correctness)}}}\label{THM:CORECTNESS}
Given a trellis, $\trellis$, dataset, $\dataset$, and objective-specific admissible heuristic function, $\Hfun$, Algorithm \ref{alg:tree-astar} yields $\tree^\star = \argmin_{\tree \in \alltrees(\trellis)} \EfunT(\tree)$
where $\alltrees(\trellis)$ is the set of all trees represented by the trellis.
\end{theorem}
\vspace{-0.4cm}
See Appendix \S\ref{prf:correctness} for proof.

\begin{corollary}{\textbf{\emph{(Optimal Clustering)}}}\label{THM:OPTIMAL}
Given a dataset, $\dataset$,let the trellis, $\trellis = \PP(\dataset)$, be the powerset. Algorithm \ref{alg:tree-astar} yields the (global) optimal hierarchical clustering for $\dataset$.
\end{corollary}
\vspace{-0.2cm}
See Appendix \S\ref{prf:optimal} for proof.

Next, we consider the space and time complexities of Algorithm \ref{alg:tree-astar}. We observe that the algorithm scales with the size of the trellis structure, and requires at most exponential space and time, rather than scaling with the number of trees, $(2n-3)!!$, which is super-exponential.

\begin{theorem}{\textbf{\emph{(Space Complexity)}}}
\label{THM:SPACE}
For trellis, $\trellis$ and dataset, $\dataset$, Algorithm \ref{alg:tree-astar} finds the lowest cost hierarchical clustering of $\dataset$ present in $\trellis$ in space $O(|\trellis|^2)$, which is at most $O(3^{|\dataset|})$ when $\trellis = \PP(\dataset)$. 
\end{theorem}
\vspace{-0.2cm}
See Appendix \S\ref{prf:space} for proof.

\begin{algorithm}[H]
 \caption{A* Hierarchical Clustering Search}
 \label{alg:tree-astar}
 \begin{algorithmic}[1]
\Function{Search}{$\trellis$, $\dataset$, $\Hfun$}
    \State \textbf{Input: } A trellis structure $\trellis$ and dataset $\dataset$, a heuristic function $\Hfun$
    \State \textbf{Output: } The lowest cost tree represented in the trellis.
    \Do
        \State \textsf{$\rhd$ Get state from trellis} (Eq. \ref{eq:tree_from_trellis})
        \State $\tree_\argmaxsplit(\dataset) = \{\dataset\} \bigcup_{\dataset_c \in \dataset[\argmaxsplit]} \tree_\argmaxsplit(\dataset_c)$ 
        \State $f_\textsf{at root}, \_, h_\textsf{at root}, \_, \_ \gets \dataset[\frontier].\textsf{peek}()$
        \State \textsf{$\rhd$ At goal state?}
        \If{$h_\textsf{at root} = 0$ \textbf{and} $|\textsf{lvs}(\tree_\argmaxsplit(\dataset))| = |\dataset|$}
        \State \textbf{return} $f_\textsf{at root},\ \tree_\argmaxsplit(\dataset)$
        \EndIf
        \State \textsf{$\rhd$ Explore new leaves}
        \For{$\dataset_i$ \textbf{in} $\textsf{lvs}(\tree_\argmaxsplit(\dataset))$}
            \For{ $\dataset_L,\ \dataset_R$ \textbf{in} $\textsc{children}(\dataset_i)$}
                 \State Define $g_{LR}$ according to Eq. \ref{eq:glr}
                 \State Define $h_{LR}$ according to Eq. \ref{eq:hlr}
                 \State $f_{LR} \gets g_{LR} + h_{LR}$
                 \State $\dataset_i[\frontier].\textsf{enqueue}((f_{LR},g_{LR},h_{LR},\dataset_L,\ \dataset_R))$
             \EndFor
        \EndFor
        \State \textsf{$\rhd$ Update each node in the tree's min heap}
        \For{$\dataset_i$ \textbf{in} $\tree_\argmaxsplit(\dataset)$ \textbf{from leaves to root}}
            \State $\_, \_, \_, \dataset_L,\ \dataset_R \gets \dataset_i[\frontier].\textsf{pop}()$
            \State Define $g\gets g_{LR}$ according to Eq. \ref{eq:glr}
            \State $h \gets 0$
            \For{$\dataset_c$ \textbf{in} $[X_L, X_R]$,}
                \If{ $\dataset_c[\frontier]$ \textbf{is defined}}
                    \State $\_, g_c, h_c, \_,\ \_ \gets \dataset_c[\frontier].\textsf{peek}()$
                    \State $g \gets g + g_c$
                    \State $h \gets h + h_c$
                \Else
                    \State $h \gets h + \Hfun(X_c)$
                \EndIf
            \EndFor
            \State \textsf{$\rhd$ Update the $f$ value of split $X_L, X_R$}
            \State $\dataset_i[\frontier].\textsf{enqueue}((g+h, g, h, X_L, X_R))$
        \EndFor
    \doWhile{True}
\EndFunction
  \end{algorithmic}
\end{algorithm}

\begin{theorem}{\textbf{\emph{Time Complexity}}}\label{THM:TIME}
For trellis, $\trellis$ and dataset, $\dataset$, Algorithm \ref{alg:tree-astar} finds the lowest cost hierarchical clustering of $\dataset$ in time $O(|\{\ \ch(\dataset, \trellis) \mid \dataset \in \trellis\}|)$. which is at most $O(3^{|\dataset|})$ when
$\trellis = \PP(\dataset)$. 
\end{theorem}
\vspace{-0.2cm}
See Appendix \S\ref{prf:time} for proof.

Finally, we observe that Algorithm \ref{alg:tree-astar} is optimally efficient when given a consistent / monotone heuristic. 

\begin{theorem}{\textbf{\emph{Optimal Efficiency}}}\label{THM:OPT_EFFICIENCY}
For trellis, $\trellis$ and dataset, $\dataset$, Algorithm \ref{alg:tree-astar} is optimally efficient if $h$ is a consistent / monotone heuristic.
\end{theorem}
\vspace{-0.2cm}
See Appendix $\S$ \ref{prf:opt_efficiency} for proof.

\section{Trellis Construction}
\label{sec:construction}
When running \astar, the exploration step at node $\dataset_i$ requires the instantiation of a min heap of node $\dataset_i$'s children.
As described in \S \ref{sec:methods}, we can use Algorithm \ref{alg:tree-astar} to find exact solutions if we search the full trellis structure that includes all subsets of the dataset, $\PP(\dataset)$.
However, we are also interested in approximate methods. 
Here, we propose three approaches: (1) run \astar using a sparse trellis (one with missing nodes and/or edges), (2) extend a sparse trellis by adding nodes, and edges corresponding to child / parent relationships while running \astar at exploration time for each node explored during \astar search, (3) run \astar iteratively, obtaining a solution and then running \astar again, extending the trellis further between or during subsequent iterations.

\paragraph{Trellis Initialization}
\label{sec:trellis_init}
Before running \astar, an input structure defining the children / parent relationships and nodes in the trellis can be provided. One possibility is to use an existing method to create this structure. For instance, it is possible to initialize all the nodes coming from the full/partial beam size set of hierarchies obtained from running beam search. However, this structure can be updated during run-time using the method described below by adding additional children to a node beyond those present at initialization. This way, \astar will include within the search space, hierarchies coming from small perturbations (at every level) of the ones employed to initialize the trellis.

\paragraph{Running \astar While Extending The Trellis}
\label{sec:on_the_fly}

A sparse trellis (even one consisting solely of the root node) can be extended during the exploration stage of A* search. Given a node, $\dataset_i$, in a sparse trellis, sample the children to place on $\dataset_i$'s queue using an objective function-based sampling scheme.
It is reasonable to randomly sample a relatively large number of candidate children of the node and then either restrict the children to the $K$ best according to the value of the $\EfunS(\cdot, \cdot)$ function of the cost (Eq. \ref{eq:tree-cost}) (best $K$ sampling) or sample from the candidate children according to their relative probability, i.e., $\EfunS(\cdot, \cdot) / \sum{\EfunS(\cdot, \cdot)}$ (importance sampling). 

\paragraph{Iterative \astar-based Trellis Construction}
\label{sec:iterative}
We can combine the methods above, by initializing a sparse trellis, extending it during run-time, and then run an iterative algorithm that outputs a series of hierarchical clusterings with monotonically decreasing cost. It uses $\trellis^{(r)}$ as the initialization, runs \astar and outputs the best hierarchical clustering represented by that trellis. In each subsequent round $r+1$, $\trellis^{(r)}$ is extended at run-time, adding more nodes and edges, and at the end of the round, outputs the best hierarchical clustering represented by trellis $\trellis^{(r+1)}$. This can be repeated until some stopping criteria are reached, or until the full trellis is created.

\section{Jet Physics Experiments}
\label{sec:jet_experiments}
\textbf{Additional Background}. Detectors at the Large Hadron Collider (LHC) at CERN measure the energy (and momentum) of particles generated from the collision of two beams of high-energy protons. Typically, the pattern of particle hits will have localized regions. Recall the particles in each region are clustered (i.e., a jet), and this hierarchical clustering is originated by a {\it showering process} where the initial (unstable) particle (root) goes through successive binary splittings until reaching the final state particles (with an energy below a given threshold) that hit the detector and are represented by the leaves. These leaves are observed while the latent showering process, described by quantum chromodynamics, is not. This showering-process is encoded in sophisticated simulators, and rather than optimizing heuristics (as is traditionally done, e.g., ~\cite{ Cacciari:2008gp,Catani:1993hr,Dokshitzer:1997in, Ellis:1993tq}), directly maximizing the likelihood, (1) allows us to compute MAP hierarchical clusterings generated by these simulators, (2) unifies generation and inference, and (3) can improve our understanding of particle physics.

\textbf{Cost Function}. We use a cost function that is the negative log-likelihood of a model for jet physics \cite{ToyJetsShower}.
Each cluster, $\dataset$, corresponds to a particle with an energy-momentum vector $x = (E \in \RR^{+}, \vec{p} \in \RR^3)$ and squared mass ${t}(x) = {E^2-|\vec{p}|^2}$. 
A parent's energy-momentum vector is obtained from adding its children, i.e., $x_P = x_L + x_R$.
We study Ginkgo, a 
prototypical
model for jet physics \cite{ToyJetsShower} that provides a tractable joint likelihood, where for each pair of parent and left (right) child cluster with masses $\sqrt{t_P}$ and $\sqrt{t_L}$ ($\sqrt{t_R}$) respectively, the likelihood function is,
    \begin{align}\label{eq:ginkgoLikelihood}
        \EfunS(X_L, X_R)= f(t(x_L) | t_P, \lambda) \cdot f(t(x_R) | t_P, \lambda)\\ \text{with} \,\,\,\,\,\,
        f(t | t_P, \lambda) = \frac{1}{1 - e^{- \lambda}} \frac{\lambda}{t_P} e^{- \lambda \frac{t}{t_P}} \label{GinkgoSplitLH}
    \end{align}
    where the first term in $f(t | t_P, \lambda)$ is a normalization factor associated to the constraint that $t<t_P$, and $\lambda$ a constant.  
For the leaves, we need to integrate $f(t | t_P, \lambda)$ from 0 to the threshold constant $t_{cut}$ (see \cite{ToyJetsShower} for more details), 
\bea\label{leaves_lh}
 f(t_{cut} | \lambda, t_{\text{P}})=\frac{1}{1-e^{- \lambda}} \bigg(1-e^{- \frac{\lambda}{ t_{\text{P}}} t_{\text{cut}}}\bigg) 
\eea
    
\textbf{Heuristic function}. We introduce two heuristic functions to set an upper bound on the log likelihood of Ginkgo hierarchies, described in detail in Appendix $\S$ \ref{app:ginkgoHeuristics}. We provide a brief description as follows. We split the heuristic into internal nodes and leaves.
For the leaves, we roughly bound $t_{\text{P}}$ in Equation \ref{leaves_lh} by the minimum squared mass, among all the nodes with two elements, that is greater than $t_{cut}$ (see \cite{ToyJetsShower} for more details).
For internal nodes, we wish to find the smallest possible upper bound to Equation \ref{GinkgoSplitLH}. We bound $t_{\text{P}}$ in the exponential by the squared mass of the root node (top of the tree).
Next, we look for the biggest lower bound on the squared mass  $t$ in Equation \ref{GinkgoSplitLH}.
We consider a fully unbalanced tree as the topology with the smallest per level values of $t$ and we bound each level in a bottom up approach (singletons are at the bottom).  
Finally, to get a bound for $t_{\text{P}}$ in the denominator of Equation \ref{GinkgoSplitLH} we want the minimum possible values. Here we consider two heuristics:
\begin{itemize}
    \item Admissible \hzero. We take the minimum per level parent squared mass plus the minimum leaf squared mass, as the parent of every internal node has one more element.
    
    \item Approximate \hone. We take the minimum per level parent squared mass plus 2 times the minimum squared mass  among all the nodes with two elements.
\end{itemize}

In principle, \hone is approximate, but we studied its effectiveness by checking that the cost was always below the exact trellis MAP tree on a dataset of 5000 Ginkgo jets with up to 9 elements, as well as the fact that exact \astar with \hone agrees with the exact trellis within 6 significant figures (see Figure \ref{fig:ginkgoCost}). In any case, if for some cases \hone is inadmissible, the only downside is that we could be missing a lower cost solution for the MAP tree. As a result, given that \hone is considerably faster than \hzero, below we show results for \astar implemented with \hone.

{\bf Data and Methods}. We compare our algorithm with different benchmarks. We start with greedy and beam search implementations as baselines, where we cluster Ginkgo jet constituents (leaves of a binary tree) based on the joint likelihood to get the maximum likelihood estimate (MLE) for the hierarchy that represents the latent structure of a jet. Greedy simply chooses the pairing of nodes that locally maximizes the likelihood at each step, whereas beam search maximizes the likelihood of multiple steps before choosing the latent path. The current implementation only takes into account one more step ahead, with a beam size given by 1000 for datasets with more than 40 elements or $\frac{N(N-1)}{2}$ for the others, with $N$ the number of elements to cluster. Also, we discarded partial clusterings with identical likelihood values, to account for the different orderings of a hierarchy (see \cite{boyles2012time} for more details), which significantly improved the performance of beam search.
As a stronger baseline, we implement the Monte-Carlo Tree Search (MCTS) algorithm for hierarchical clusterings introduced in \cite{Brehmer:2020brs}. We choose the best performing case, that corresponds to a neural policy trained for 60000 steps, initializing the search tree at each step by running beam search with beam size 100. Also, the maximum number of MCTS evaluation roll-outs is 200.
Finally, for exact solutions, we compare with the exact hierarchical cluster trellis algorithm introduced in \cite{greenberg2020compact}. 

\textbf{Approximate \astar}. We design the approximate \astar-based method following the details described in Section \ref{sec:construction}. We first initialize the sparse trellis the full set of beam search hierarchies (\S \ref{sec:trellis_init}). Next, we use a top K-based sampling (\S \ref{sec:on_the_fly}) and the iterative construction procedure (\S \ref{sec:iterative}) to extend the search space.

Figure \ref{fig:ginkgoCost} shows a comparison of the MAP values of the proposed exact and approximate algorithms using \astar with benchmark algorithms (greedy, beam search, MCTS, and exact trellis), on Ginkgo jets. For the \astar-based method we show both the exact algorithms and approximate solutions, both implemented with the heuristic denoted as \hone in Appendix $\S$ \ref{app:ginkgoHeuristics}. 
We want to emphasize that approximate versions of \astar allow the algorithm to handle much larger datasets while significantly improving over beam search and greedy baselines. MCTS provides a strong baseline for small datasets, where it is feasible to implement it. However, \astar shows an improvement over MCTS while also being feasible for much larger datasets.
\begin{figure}
    \centering
    \includegraphics[width=0.65\textwidth]{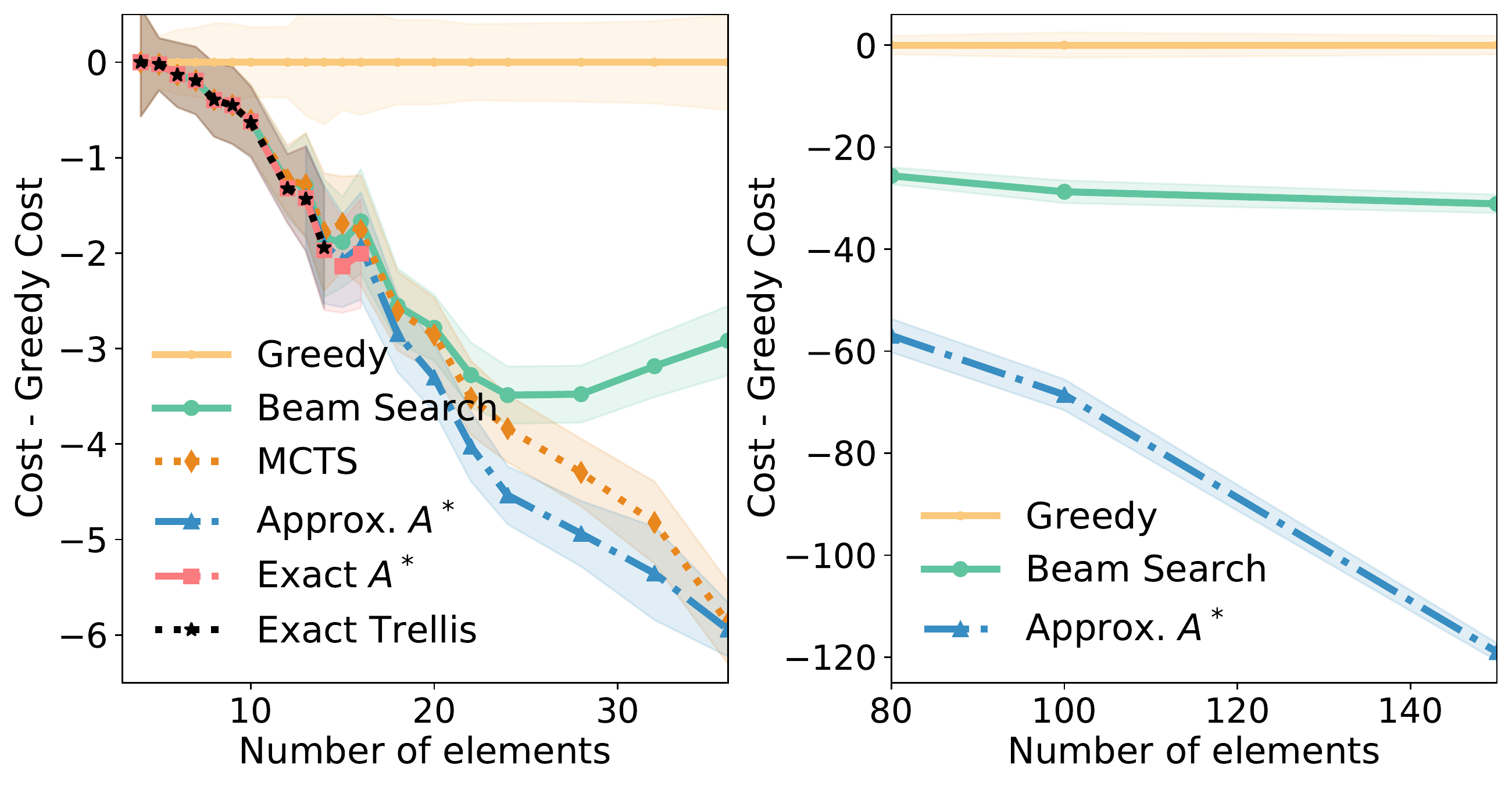}
    \vspace{-0.3cm}
    \caption{\small{\textbf{Jet Physics Results.} Cost (Neg. log. likelihood)  for the MAP hierarchy of each algorithm on Ginkgo datasets minus the cost for greedy (lower values are better solutions). We see that both exact and approx. \astar greatly improve over beam search and greedy. Though MCTS provides a strong baseline for small datasets where it is feasible to implement it (left), \astar still shows an improvement over it.}}
    \label{fig:ginkgoCost}
\end{figure}
Next in Figure \ref{fig:ginkgo_runtime}  we show the empirical running times.
We want to point out that the \astar-based method becomes faster than the exact trellis one for data with 12 or more elements, which makes \astar feasible for larger datasets. 
We can see that approx. \astar follows the exact \astar cost very closely starting at 12 elements until 15, while having a lower running time. Though approx. \astar and MCTS have a running time of the same order of magnitude (only evaluation time for MCTS) for more than 20 leaves, \astar has a controlled running time while MCTS evaluation time grows exponentially.
\begin{figure}[!htbp]
    \centering
    \includegraphics[width=0.65\textwidth]{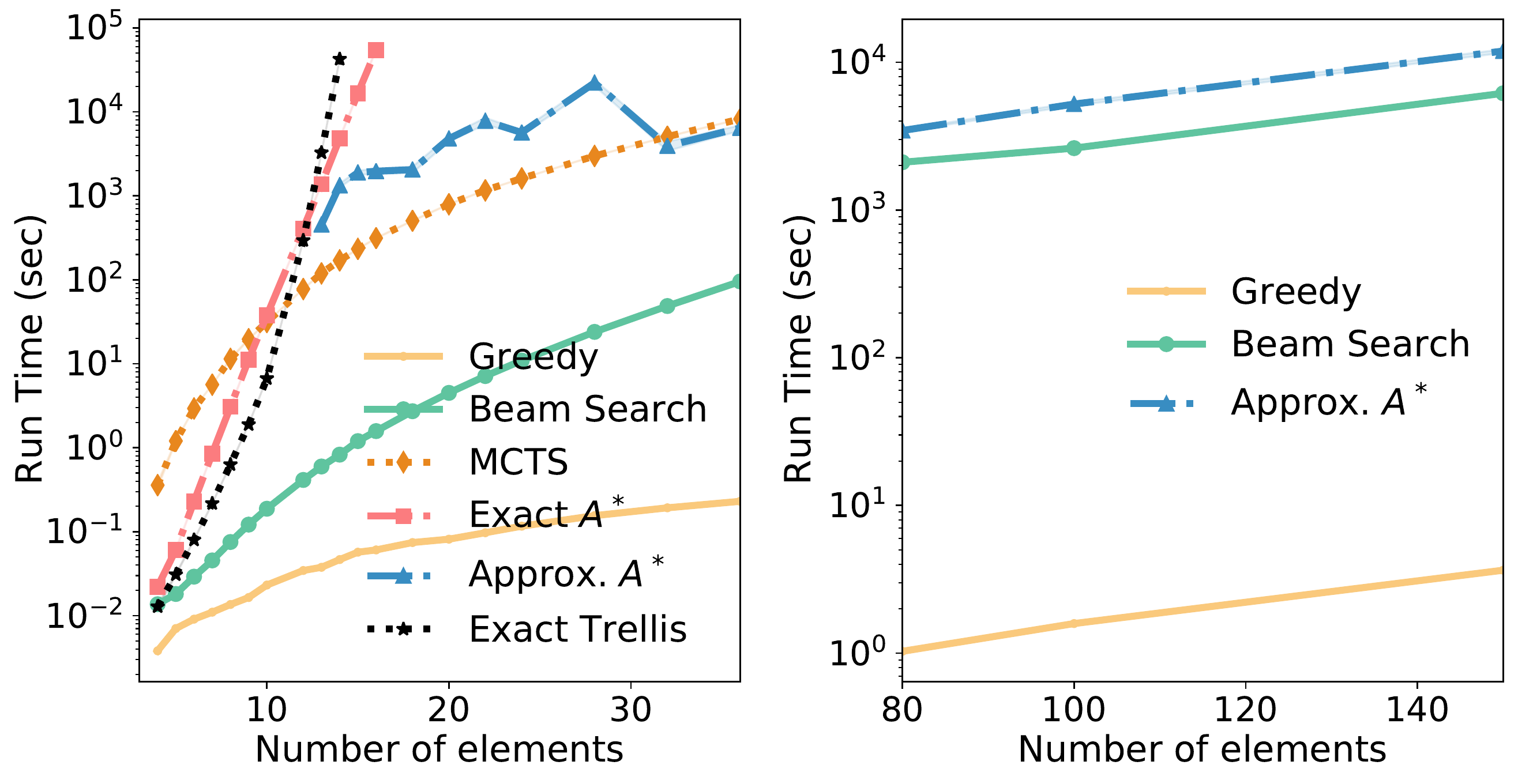}
      \vspace{-0.3cm}
    \caption{\small{\textbf{Running Times.} Empirical running time on an Intel Xeon Platinum 8268 24C 2.9GHz Processor. The \astar-based method becomes faster than the exact trellis one for data with 12 or more elements. Approx. \astar is much faster than exact algorithms while still significantly improving over the baselines. This figure only shows evaluation running times for MCTS, after having trained the model for 7 days with the same processor. Evaluation MCTS times grow exponentially while approx. \astar has a controlled run time.}}
    \label{fig:ginkgo_runtime}
\end{figure}

The exact trellis time complexity is $\mathcal{O}(3^N)$ which is super-exponentially more efficient than brute force methods that consider every possible hierarchy and grow at the rate of $(2N-3)!!$. 
Thus, in Figure \ref{fig:ginkgo_complexity} we show the number of hierarchies explored by \astar divided by $3^N$, and we can see that exact and approximate \astar are orders of magnitude more efficient than the trellis, which to the best of our knowledge was the most efficient exact algorithm at present. 
Also, we note as a point of reference that for about 100 leaves, the approx. \astar results of Figure \ref{fig:ginkgo_complexity} can improve over benchmarks by only exploring  $\mathcal{O}(10^{-180})$ of the possible hierarchies.
\begin{figure}
    \centering
    \includegraphics[width=0.65\textwidth]{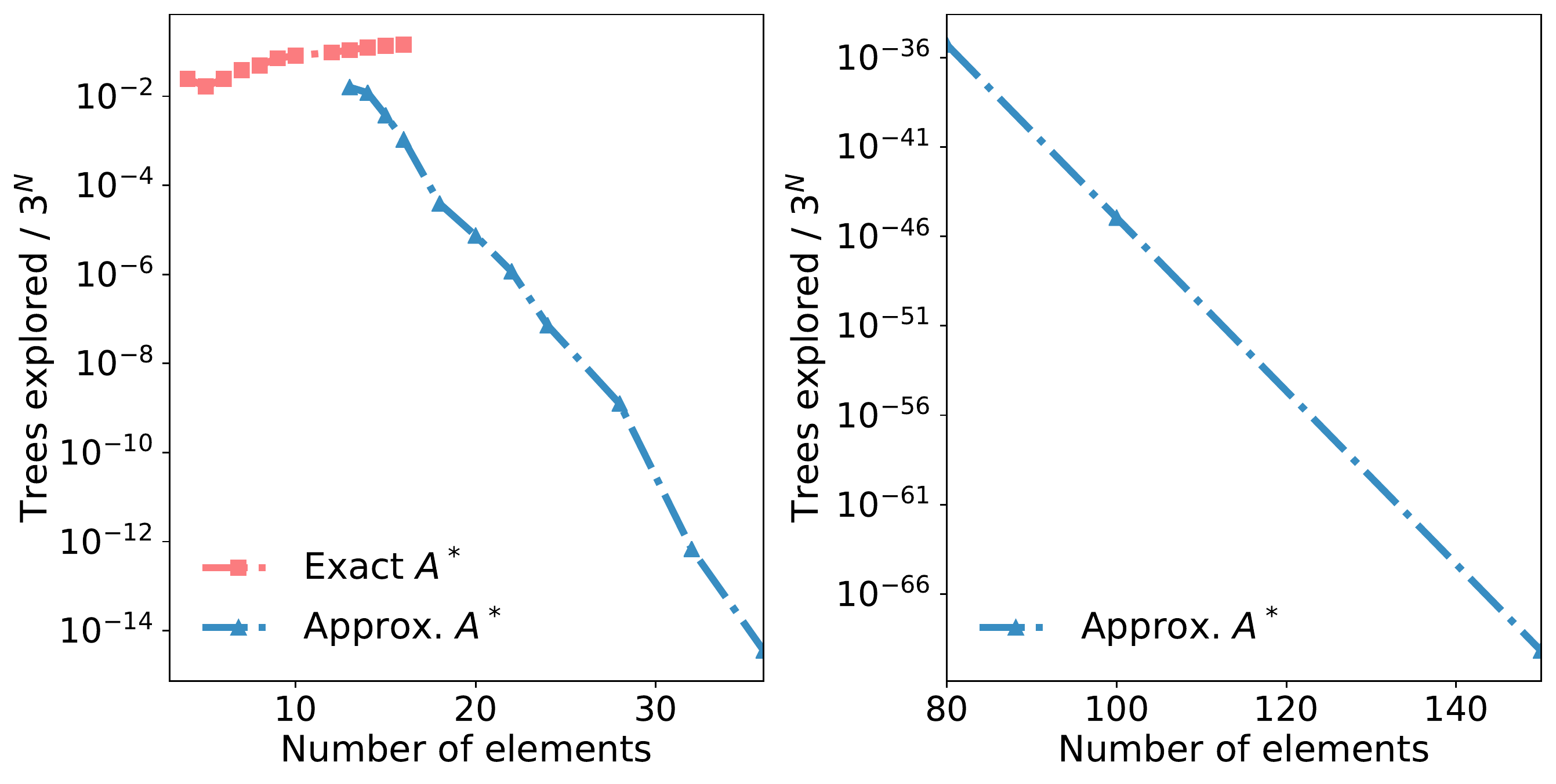}
      \vspace{-0.3cm}
    \caption{\small{\textbf{Search Space Exploration.} Number of hierarchies explored divided by the exact trellis time complexity, i.e $3^N$. We see that the \astar-based method is considerably more efficient than the trellis, which, in turn, is super-exponentially more efficient than an explicit enumeration of clusterings.}}
    \label{fig:ginkgo_complexity}
\end{figure}

\begin{figure*}
    \centering
    \includegraphics[width=0.24\textwidth]{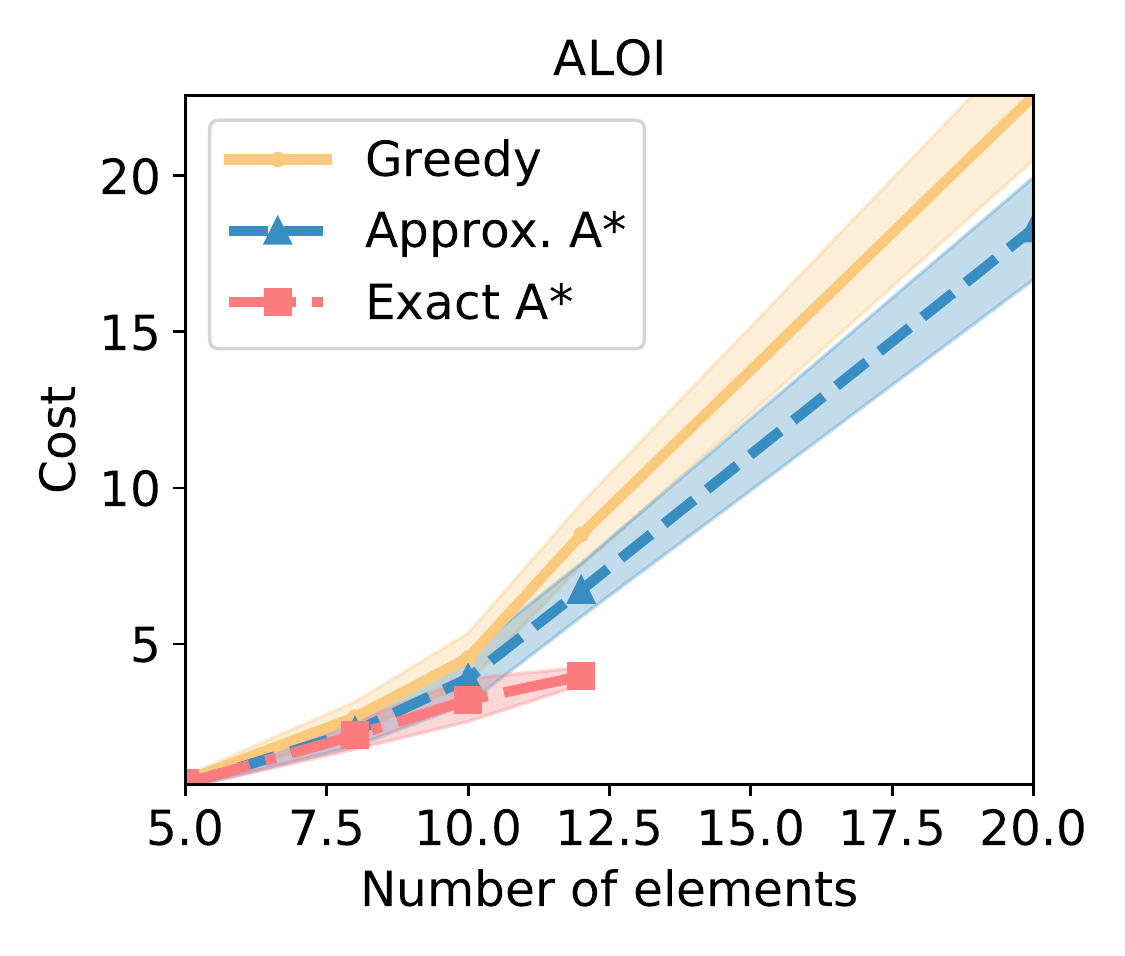}
    \includegraphics[width=0.24\textwidth]{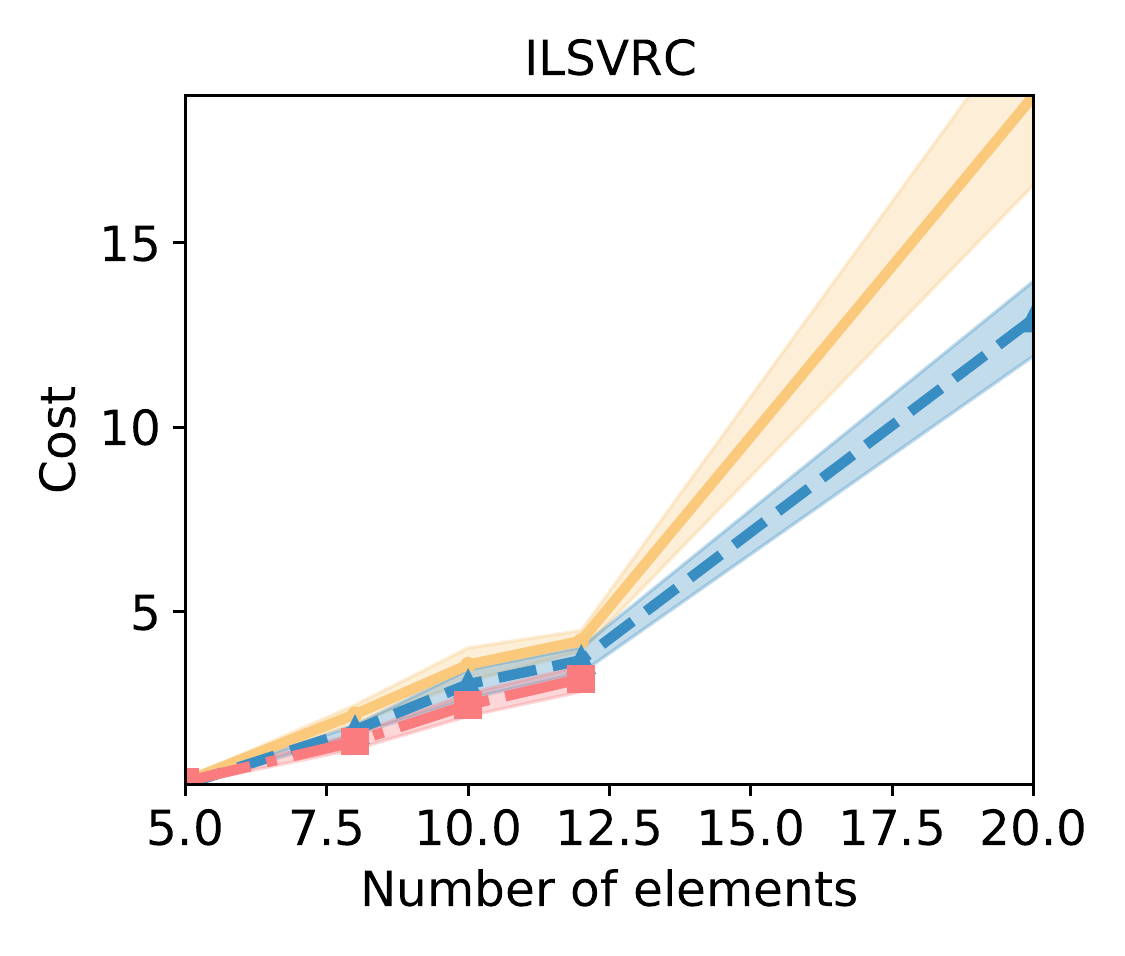}
    \includegraphics[width=0.24\textwidth]{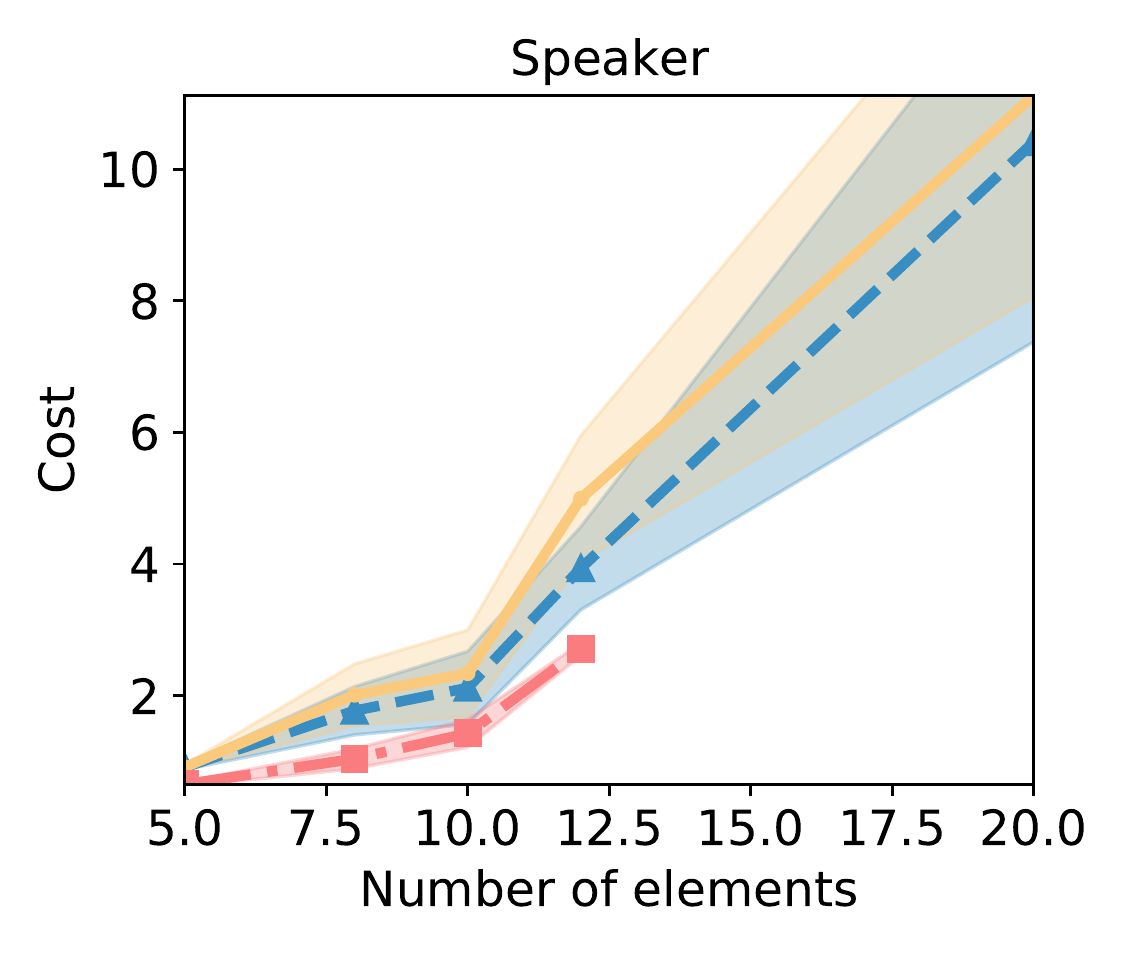}
    \includegraphics[width=0.24\textwidth]{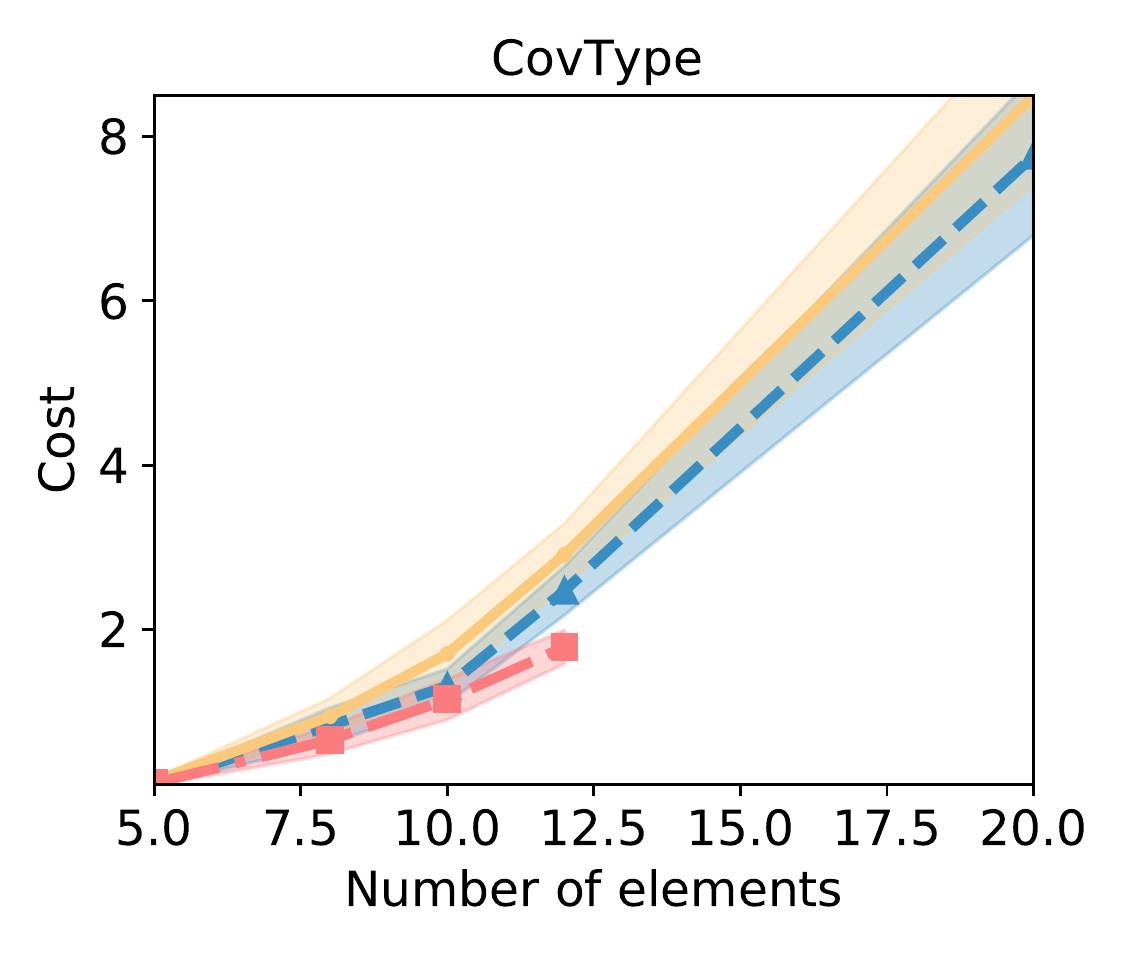}
    \caption{\textbf{Hierarchical Correlation Clustering on Benchmark Data}. On randomly sampled subsets small enough for us to run the exact \astar approach, we report the hierarchical correlation clustering cost (lower is better) for clustering benchmarks. We observe that the approximate \astar method is able to achieve lower cost clusterings compared to greedy.}
    \label{fig:benchmarks}
\end{figure*}

\begin{figure*}
    \centering
    \includegraphics[width=0.24\textwidth]{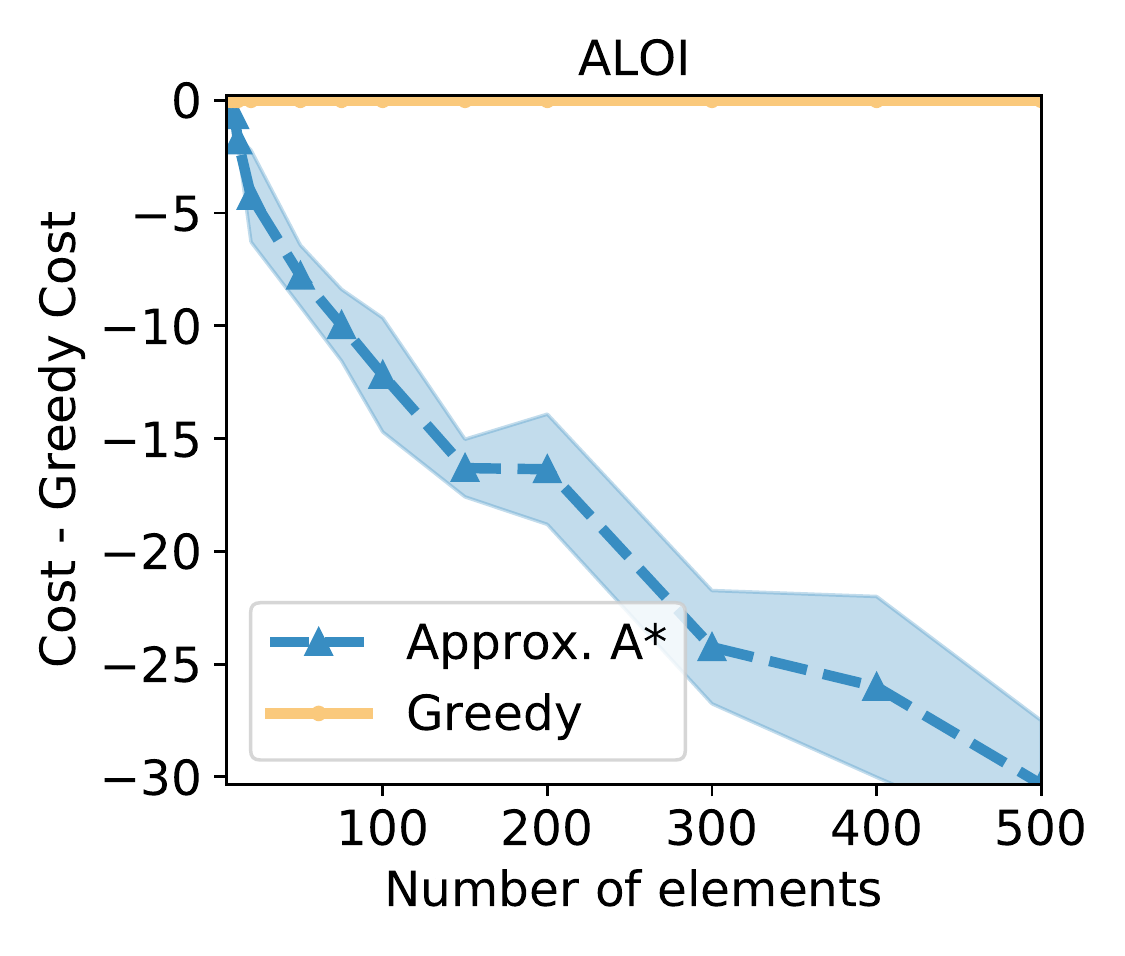}
    \includegraphics[width=0.24\textwidth]{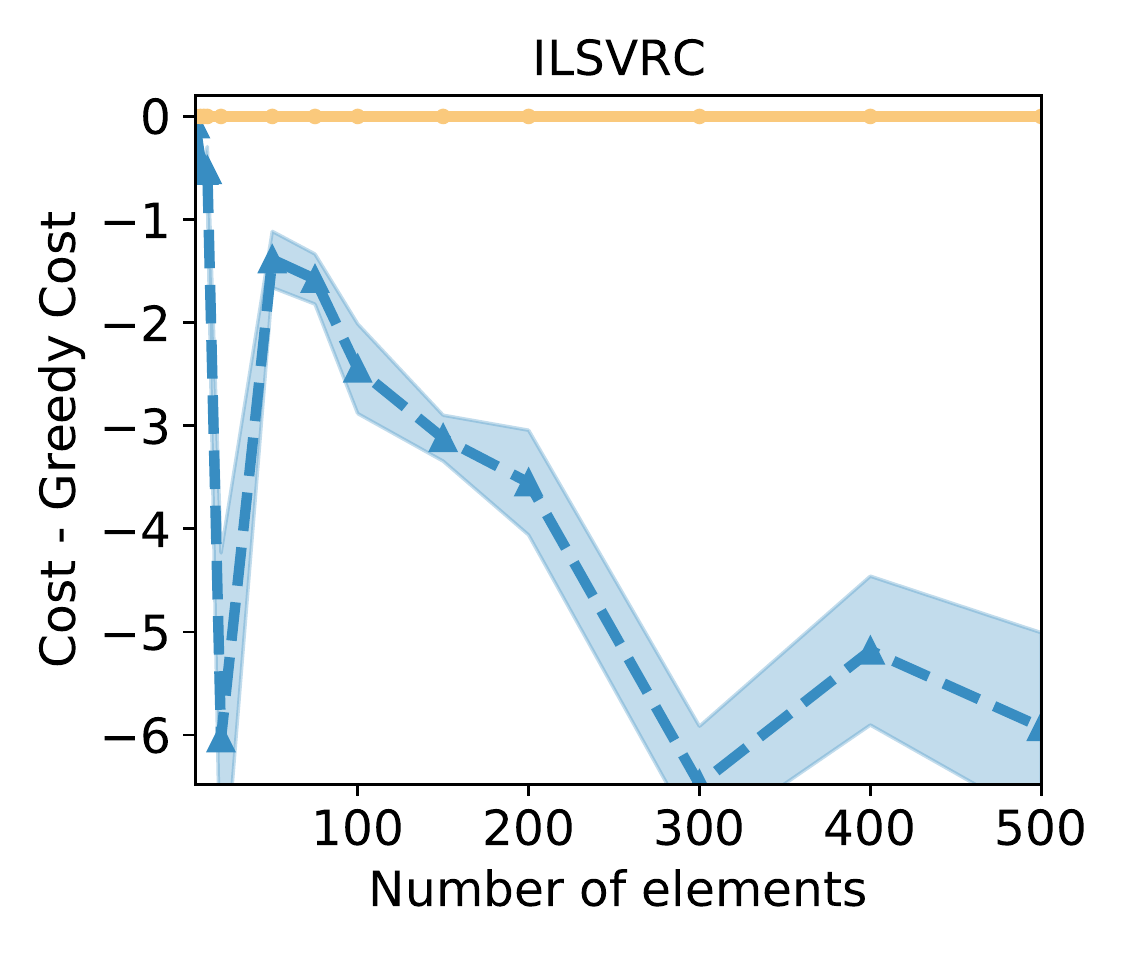}
    \includegraphics[width=0.24\textwidth]{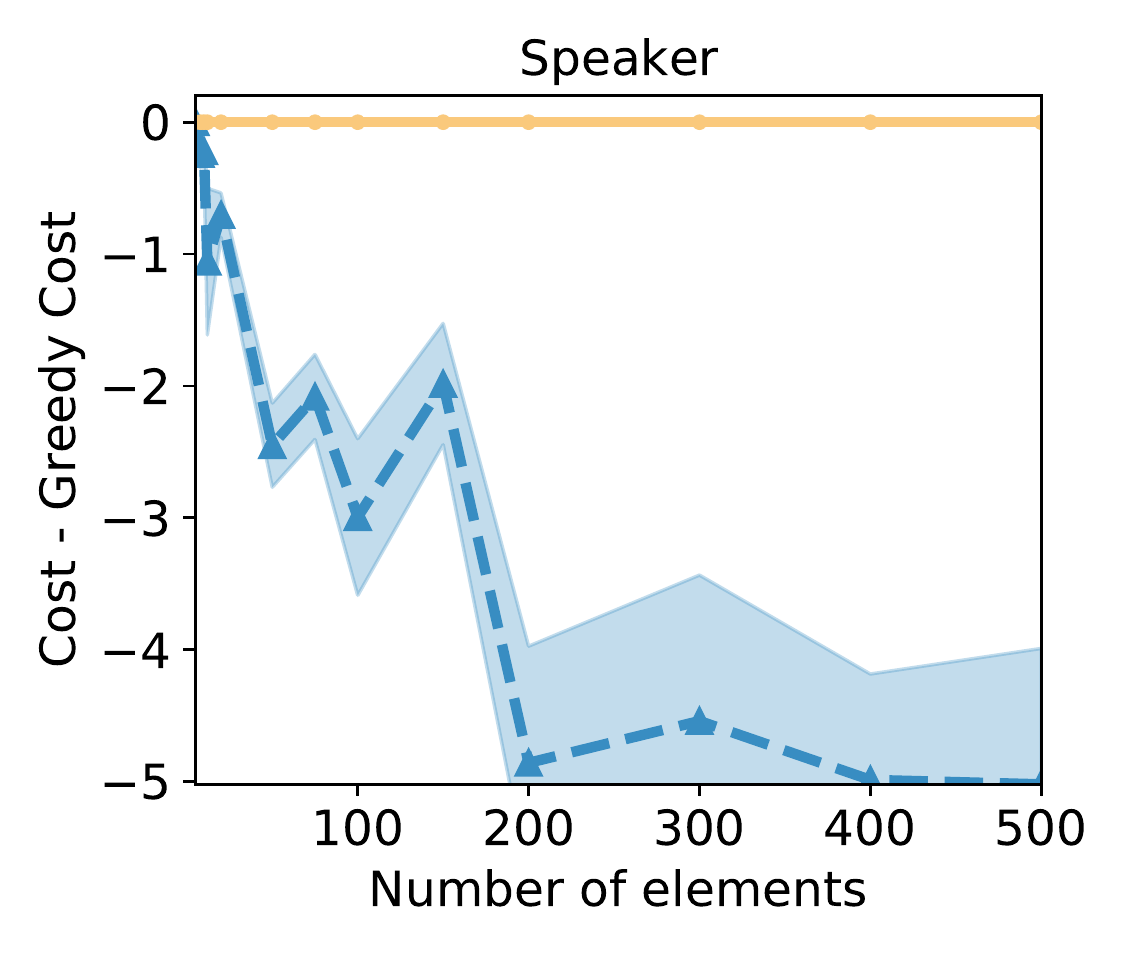}
    \includegraphics[width=0.24\textwidth]{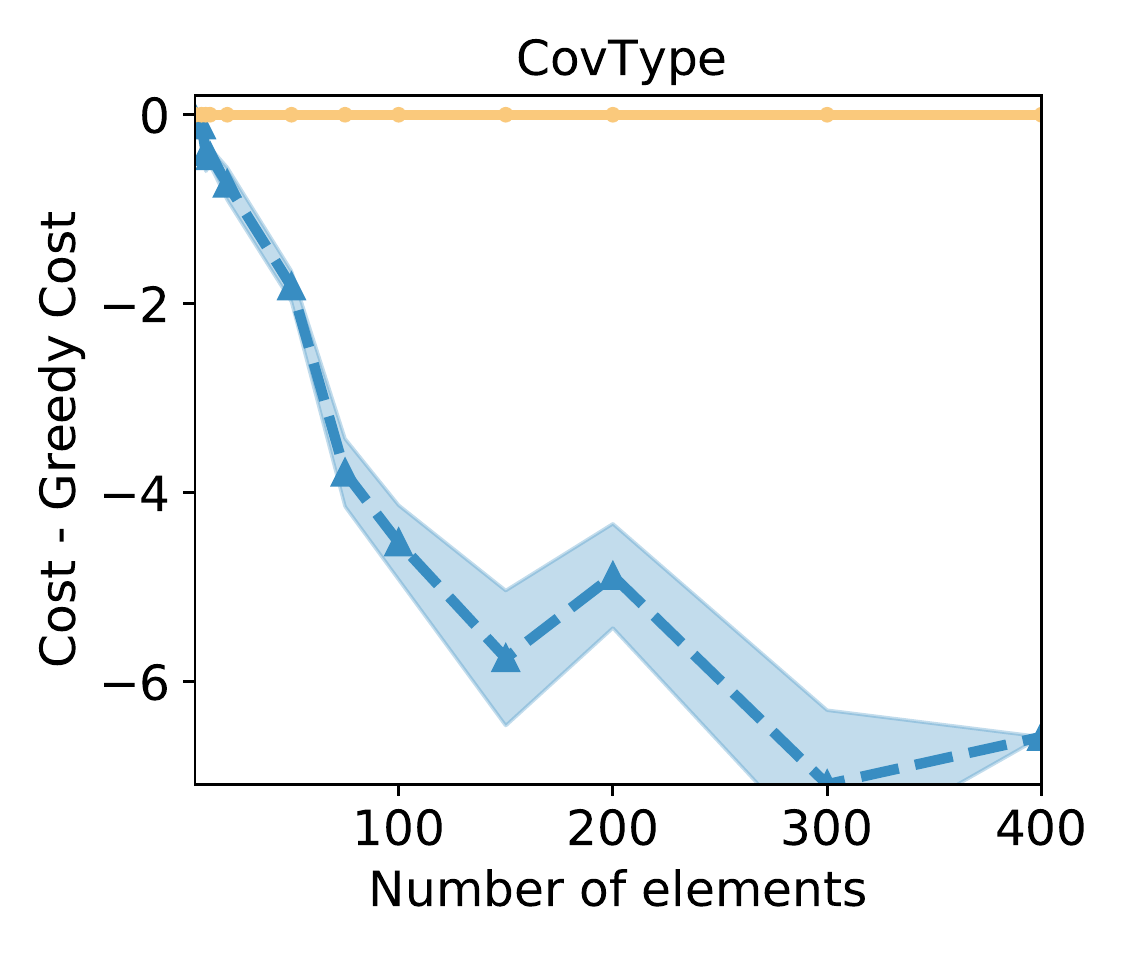}
    \caption{\textbf{Improvement over Greedy on Larger Benchmark Datasets }. We report the improvement (cost, lower is better) of the approximate \astar method over the greedy baseline. }
    \label{fig:larger}
\end{figure*}

\section{Experiments on Clustering Benchmarks}
\label{sec:benchmarks}
In this section, we analyze the performance of \astar
on several commonly used hierarchical clustering benchmark datasets
\cite{kobren2017hierarchical,monath2019scalable}:  
    \textbf{Speaker},  i-vector speaker recordings, ground truth clusters refer each unique speaker \cite{greenberg2014nist}; \textbf{CovType}, a dataset of forest cover types;
    \textbf{ALOI}, color histogram features of 3D objects, ground truth clusters refer to each object type \cite{geusebroek2005amsterdam};
     \textbf{ILSVRC}, image features from InceptionV3 \cite{szegedy2016rethinking}  of the ImageNet ILSVRC 2012 dataset \cite{russakovsky2015imagenet}.  
     
In this experiment, we use cosine similarity as the measure of the pairwise similarity between points, as it is known to be meaningful for each of the datasets \cite{monath2019scalable}, and hierarchical correlation clustering as 
the cost function: 

\begin{objective}
\label{ex:hcc}
\textbf{\emph{(Hierarchical Correlation Clustering)}}
Following \cite{greenberg2020compact}, we consider a hierarchical version of well-known flat clustering objective, correlation clustering  \cite{bansal2004correlation}. 
 In this case, we define the cost of a pair of sibling nodes in the tree to be the sum of the positive edges crossing the cut, minus the sum of the negative edges not crossing the cut: 
{\small
\begin{align}
    \EfunS(X_i,X_j) = 
    \sum_{\mathclap{\substack{x_i, x_j \in X_i \times X_j}}} w_{ij} \II[w_{ij} > 0] & + 
    \sum_{\mathclap{\substack{x_i,x_j \in X_i \times X_i, \\ i < j}}} |w_{ij}| \II[w_{ij} < 0]  +
    \sum_{\mathclap{\substack{x_i,x_j \in X_j \times X_j, \\ i < j}}} |w_{ij}| \II[w_{ij} < 0]
\end{align}
}
where $w_{ij}$ is the affinity between $x_i$ and $x_j$. 
\end{objective}

To build the weighted graphs needed as input to correlation clustering, we subtract the mean similarity from each of the pairwise similarities. 

We can provide a lower bound for this objective  by using the sum of the of the positive edges contained in the dataset,
\begin{equation}\label{eq:hcc_heuristic1}
    h_\textsf{cc}(\dataset) = 
    \sum_{\mathclap{\substack{x_i,x_j \in X \times X, \\ i < j}}} w_{ij} \II[w_{ij} > 0].
\end{equation}

\begin{proposition}{\textbf{\emph{(Admissible Heuristic for Hierarchical Correlation Clustering)}}}\label{THM:HCC_HEURISTIC1}
Equation \ref{eq:hcc_heuristic1} is an admisslbe heuristic for Hierarchical Correlation Clustering cost, that is $ h_\textsf{cc}(\dataset) \leq \EfunT_\textsf{HCC}(\dataset)$.
\end{proposition}
\vspace{-0.2cm}
See Appendix \S\ref{prf:hcc-heuristic1} for proof.

\textbf{Approximate \astar}. In this experiment, we apply the techniques described in Section \ref{sec:construction} to design an approximation algorithm. We first initialize the sparse trellis using a greedy approach (\S \ref{sec:trellis_init}). We then use a top K-based sampling (\S \ref{sec:on_the_fly}) and iterative construction approach (\S \ref{sec:iterative}) to extend the search space beyond greedy initialization and use heuristic $h_\textsf{cc}$.

Figure \ref{fig:benchmarks}, shows the cost (lower is better) of the exact solution found via \astar, the aforementioned approximate \astar, and the greedy solution on small datasets. We sample each of the small datasets from their corresponding original dataset using stratified sampling, where the strata correspond to the ground truth class labels. We report the mean and standard deviation / sqrt(num runs) computed over 5 random samples of each dataset size. We observe that the approximate method is able to provide 
hierarchical clusterings that have lower
cost than the greedy approach, nearing
the exact MAP values found by the optimal approach. In Figure \ref{fig:larger}, we show results on  data sets with a larger number of samples (sampled from the original datasets in the same manner as before), which are too large to run the exact algorithm. We plot the reduction (mean, standard deviation / sqrt(num runs) over 5 runs per dataset size) in cost of the clusterings found by the approximate \astar algorithm vs. the greedy approach. 
\vspace{-0.4cm}

\section{Related Work}
An \astar-based approach is used in \cite{daume2007fast} to 
find MAP solutions to Dirichlet Process Mixture Models. This work focuses on flat clustering rather than hierarchical, and 
it does not include a detailed analysis of the time and space complexity, as is done in this paper. 
In contrast to our deterministic search-based 
approach, sampling-based methods, such as \cite{bouchard2012phylogenetic,knowles2011pitman,ghahramani2010tree,boyles2012time}, use MCMC, Metropolis Hastings or Sequential Monte Carlo methods to draw samples from a posterior distribution over trees.
Other methods use incremental constructions of tree structures %\kyle{reference\craig{fixed}}
\cite{kobren2017hierarchical,monath2019scalable,zhang1996birch}, similar to our approach, however they perform tree re-arrangements to overcome greedy decisions made in the incremental setting. These operations are made in a relatively local fashion without considering a global cost function for tree structures as is done in our paper. 
Linear, quadratic, and semi-definite 
programming have also been used to solve hierarchical clustering
problems \cite{gilpin2013formalizing,roy2017hierarchical}, as have, branch-and-bound based methods \cite{cotta2003memetic}. These methods often have approximation guarantees. 
It would be interesting to discover how these methods could
be used to explore nodes in a trellis structure (\S \ref{sec:on_the_fly}) in future work. 
In contrast to our search-based method that considers discrete tree structures, recent work has considered continuous representations of tree structure to support gradient-based optimization \cite{monath2019gradient,chami2020trees}. 
{Finally, we note that much recent work,  \cite{charikar2017approximate,charikar2019hierarchical,chatziafratis2018hierarchical,cohen2017hierarchical,cohen2019hierarchical,roy2017hierarchical}, has been devoted to Dasgupta's cost for hierarchical clustering \cite{dasgupta2016cost}. We direct the interested reader to the Appendix \S \ref{app:dasgupta}, for a description of the cost and an admissible heuristic that enables the use of \astar to find optimal Dasgupta cost hierarchical clusterings.
}
\vspace{-0.4cm}

\section{Conclusion}
In this paper, we present a way to run \astar to provide
exact and approximate hierarchical clusterings. We describe
a data structure with a nested representation of the \astar 
search space to provide a time and space efficient approach 
to search. We demonstrate the effectiveness of the approach 
experimentally and prove theoretical results about its optimality
and efficiency. In future work, we hope to extend the same
kinds of algorithms to flat clustering costs as well as to probabilistic
models in Bayesian non-parametrics.

\section*{Acknowledgements}

Kyle Cranmer and Sebastian Macaluso are supported by the National Science Foundation under the awards ACI-1450310 and OAC-1836650 and by the Moore-Sloan data science environment at NYU. Patrick Flaherty is supported in part by NSF HDR TRIPODS award 1934846. Andrew McCallum and Nicholas Monath are supported in part by the Center for Data Science and the Center for Intelligent Information Retrieval, and in part by the National Science Foundation under Grant No. NSF-1763618. Any opinions, findings and conclusions or recommendations expressed in this material are those of the authors and do not necessarily reflect those of the sponsor.

\bibliography{references}
\clearpage
\appendix
\section{Appendix}
\subsection{Theorem \ref{THM:CORECTNESS} (Correctness of A* MAP) Proof}
\label{prf:correctness}
\begin{proof}
Assume toward contradiction that Algorithm \ref{alg:tree-astar} does not return the optimal hierarchical clustering, $\tree^\star$. 
It is clear from the conditional in line 9, $|\textsf{lvs}(\tree_\argmaxsplit(\dataset))| = |\dataset|$ 
and the fact that if $\dataset_L, \dataset_R$ are children of $\dataset$ then $\dataset_L \bigcup \dataset_R = \dataset$ and $\dataset_L \bigcap \dataset_R = \varnothing$, 
that Algorithm \ref{alg:tree-astar} returns a hierarchical clustering, $\tree_\argmaxsplit(\dataset) \in \alltrees(\trellis)$, so the assumption is that $\tree_\argmaxsplit(\dataset)$ is not optimal, i.e., $\EfunT(\tree^\star) < \EfunT(\tree_\argmaxsplit(\dataset))$.  
Consider the set, $S$, of data sets present in both $\tree^\star$ and $\tree_\argmaxsplit(\dataset)$ that branch to different children, i.e., $S = \{X_i | X_{L^\star},X_{R^\star} \in \tree^\star \land 
X_L,X_R \in \tree_\argmaxsplit(\dataset) \land 
X_{L^\star} \bigcup X_{R^\star} = X_L \bigcup X_R = X_i \land X_{L^\star} \neq X_L \neq X_R \neq X_{R^\star} \}$.
The set $S$ must be non empty, otherwise $\tree_\argmaxsplit(\dataset) = \tree^\star$, a contradiction.
Now consider the set, $T$, of ancestors of $S$, i.e., $T = \{ \dataset_i | \dataset_i \in S \land \nexists X_j \in S$ s.t. $X_i \neq X_j \land X_i \bigcap X_j = X_i \}$.
Since $S$ is non-empty, the set $T$ must also be non empty.
Now select a dataset, $X_i \in T$, such that the sub-hierarchy rooted at $X_i$ in $\tree^\star$,  $\tree^\star(\dataset_i)$, has lower energy than the sub-hierarchy rooted at $X_i$ in $\tree_\argmaxsplit(\dataset)$, $\tree_\argmaxsplit(\dataset_i)$, i.e., $\EfunT(\tree^\star(\dataset_i)) < \EfunT(\tree_\argmaxsplit(\dataset_i))$. 
There must exist at lease one such $X_i$, otherwise $\EfunT(\tree_\argmaxsplit(\dataset)) <= \EfunT(\tree^\star)$, a contradiction.

Now consider the trellis vertex corresponding to data set $\dataset_i$, $\trellisVertex_i$, with min heap $\dataset_i[\frontier]$, and the tuple at the top of $\dataset_i[\frontier]$, $(f_i, g_i, h_i, X_L, X_R)$. Note that because $h=0$ in line 9, that $h_i = 0$, therefore the $f_i$ is equal to the energy of the sub-hierarchy of $\tree_\argmaxsplit(\dataset)$ rooted at $\dataset_i$, $\EfunT(\tree_\argmaxsplit(\dataset_i))$.
Note also that $\dataset_i$'s children in $\tree^\star$, $ X_{L^\star},X_{R^\star}$, must be present in $\dataset_i[\frontier]$, along with corresponding $f^\star, g^\star, h^\star$ values, otherwise $\tree^\star$ is not represented in the trellis, a contradiction.
Since $X_{L^\star} \neq X_L \neq X_R \neq X_{R^\star}$, the tuple $(f^\star, g^\star, h^\star, X_{L^\star},X_{R^\star})$ is not at the top of $\dataset_i[\frontier]$. 
If $h^\star = 0$, then $f^\star = \EfunT(\tree^\star(\dataset_i))$ and $f <= f^*$, which implies $\EfunT(\tree_\argmaxsplit(\dataset_i)) <= \EfunT(\tree^\star(\dataset_i))$, a contradiction.
If $h^\star \neq 0$, then $f^\star <= \EfunT(\tree^\star(\dataset_i))$, otherwise $\Hfun$ is not an admissible heuristic, a contradiction.
This gives us 
$\EfunT(\tree_\argmaxsplit(\dataset_i)) = f < f^\star <= \EfunT(\tree^\star(\dataset_i))$, a contradiction.
\end{proof}

\subsection{Corollary \ref{THM:OPTIMAL} (Optimal Clustering) Proof}
\label{prf:optimal}
\begin{proof}
Theorem \ref{THM:CORECTNESS} states that Algorithm \ref{alg:tree-astar} returns the optimal hierarchical clustering, thus proving that all hierarchical clusterings are represented in a full trellis would prove the corollary.

Assume toward contradiction that there is a hierarchical clustering, $\tree^\star$, that is not present in a full trellis, $\trellis$.  
This implies either that (1) there is a data set $\dataset_i \in \tree^\star$ that is not in the trellis, i.e., $\dataset_i \not\in \trellis$, or that (2) there exists data sets $\dataset_i, \dataset_j, \dataset_k \in \tree^\star$ s.t. $\dataset_i = \dataset_j \bigcup \dataset_k \land \dataset_j \bigcap \dataset_k = \varnothing$ and $\dataset_i, \dataset_j, \dataset_k \in \trellis$, but $\dataset_j, \dataset_k \not\in \children(\dataset_i)$. 
Regarding (1), $\dataset_i \in \tree^\star$ s.t. $\dataset_i \not\in \trellis$ implies that $\dataset \not\in \powerset(\dataset)$, a contradiction.
Regarding (2), since each node in a full trellis has all possible children, i.e., $\forall \trellisVertex_i \in \trellis, \children(\trellisVertex_i) = \powerset(\dataset_i)$, this implies that $\dataset_j, \dataset_k \not\in \powerset(\dataset_i)$, a contradiction.

\end{proof}

\subsection{Theorem  \ref{THM:SPACE} (Space Complexity of A* MAP) Proof}
\label{prf:space}
\begin{proof}
Each vertex, $\trellisVertex_i$, in a trellis, $\trellis$, stores a dataset, $\dataset_i$, a min heap, $\dataset_i[\frontier]$, and a pointer to the children associated with the best split, $\dataset[\argmaxsplit]$.  Since $\dataset_i$ and each pointer in $\dataset[\argmaxsplit]$ can be represented as integers, they are stored in $O(1)$ space.  The min heap, $\dataset_i[\frontier]$, stores a five tuple, $(f_{LR},\ g_{LR},\ h_{LR}, \dataset_L,\ \dataset_R)$ for each of $\dataset_i$'s child pairs, $(\dataset_L, \dataset_R)$. Each five tuple can be stored in $O(1)$ space, so the min heap is stored in $O(|\children(\dataset_i)|)$ space, where $\children(\dataset_i)$ is $\dataset_i$'s children. Thus each vertex, $\trellisVertex_i$, takes $O(|\children(\dataset_i)|)$ space to store.  In a full trellis,$|\children(\dataset_i)| = 2 ^ {|\dataset_i| -1} - 1$, making the total space complexity $\sum_{\trellisVertex \in \trellis}2^{|\dataset_i|} = O(3^{|\dataset|})$. In a sparse trellis, the largest number of parent/child relationships occurs when the trellis is structured such that the root, $\trellisVertex_1$, is a parent of all the remaining $|\trellis| -1$ vertices in the trellis, the eldest child of the root, $\trellisVertex_2$, is a parent of all the remaining $|\trellis| -2$ vertices in the trellis, and so on, thus the space complexity of a sparse trellis is $\sum_{\trellisVertex_i \in \trellis}O(|\children(\dataset_i)|) <= \sum_{i = 1...|\trellis|}{(|\trellis| -i)} = O(|\trellis|^2)$.
\end{proof}

\subsection{Theorem \ref{THM:TIME} (Time Complexity of A* MAP) Proof}
\label{prf:time}
\begin{proof}
During each iteration of the loop beginning at line 4, the algorithm descends from the root down to the leaves of the partial hierarchical clustering with minimum energy, $\tree_\argmaxsplit(\dataset)$, such that the internal nodes have all been explored and the leaf nodes have not. Upon arriving at the leaves of the partial hierarchical clustering (line 11), each leaf node is explored, creating and populating a min heap stored at the node. Finally in line 19, every node in $\tree_\argmaxsplit(\dataset)$ updates its min heap.

Note that when running Algorithm \ref{alg:tree-astar}, each node in the trellis is explored at most once (a given node, $\trellisVertex_i$, is explored when $\trellisVertex_i \in lvs(\dataset[\argmaxsplit])$ at line 12).
If every node in a trellis is explored, $\tree_\argmaxsplit(\dataset) = \argmin_{\tree \in \alltrees(\dataset)}$, 
at which point Algorithm \ref{alg:tree-astar} computes $\tree_\argmaxsplit(\dataset)$ at line 6 in $O(\log n)$ time and halts at line 10.
Therefore, the time it takes Algorithm \ref{alg:tree-astar} to explore every node in trellis $\trellis$ gives an upper bound on the time it takes Algorithm \ref{alg:tree-astar} to find the tree with minimum energy hierarchical clustering encoded by $\trellis$. 

The time it takes Algorithm \ref{alg:tree-astar} to explore every $\trellisVertex_i \in \trellis$ is the sum of (1) the time $\trellisVertex_i$ contributes to computation of $\tree_\argmaxsplit(\dataset)$ at line 6 when $\trellisVertex_i \in lvs(\dataset[\argmaxsplit])$, (2) the time it takes to populate the min heap for each $\trellisVertex_i$ at line 13, and (3) the time $\trellisVertex_i$ contributes to updating the min heap of every node in $\tree_\argmaxsplit(\dataset)$ at line 19 when $\trellisVertex_i \in lvs(\dataset[\argmaxsplit])$.

(1) The time it takes for Algorithm \ref{alg:tree-astar} to compute $\tree_\argmaxsplit(\dataset)$ in a given iteration is the sum of the length of the root to leaf paths in $lvs(\dataset[\argmaxsplit])$, since reading \dataset[\argmaxsplit] takes $O(1)$ time. 
Therefore the amount of time $\trellisVertex_i$ contributes to computation of $\tree_\argmaxsplit(\dataset)$ is the length of the path in $\tree_\argmaxsplit(\dataset)$ from the root to $\trellisVertex_i$. 
In a full trellis, the maximum possible path length from the root to $\trellisVertex_i$ is $|\dataset| - |\dataset_i|$. Thus, the total amount of time spent computing $\tree_\argmaxsplit(\dataset)$ at line 6 in a full trellis when every vertex is explored is $\sum_{k=1...n}{\binom{n}{k}(n-k)} = \frac{1}{2} \frac{(2^{n}-2)}{n} = O(2^n)$. 
In a sparse trellis, the path lengths are maximized when the root, $\trellisVertex_1$, is a parent of all the remaining $|\trellis| -1$ vertices in the trellis, the eldest child of the root, $\trellisVertex_2$, is a parent of all the remaining $|\trellis| -2$ vertices in the trellis, and so on. 
Thus the sum of the total path lengths is bounded from above by $\sum_{i = 1...|\trellis|}{(|\trellis| -i)} = O(|\trellis|^2)$, and the total amount of time spent computing $\tree_\argmaxsplit(\dataset)$ at line 6 in a sparse trellis when every vertex is explored is $O(|\trellis|^2)$.

(2) The time it takes for Algorithm \ref{alg:tree-astar} to populate the min heap for $\trellisVertex_i$ is the amount of time it takes to compute $f_{LR}, g_{LR}$, $h_{LR}$ (lines 14-16) and to enqueue $(f_{LR}, g_{LR}, h_{LR}, \dataset_L, \dataset_R)$ onto the min heap (line 17) for all children, $\dataset_L, \dataset_R$.

        It is possible to compute $g_{LR}$ in $O(1)$ time, since the first term in equation \ref{eq:glr} is a value look up in the model, and second and third terms are the $g_{LR}$ values at $\dataset_L$ and $\dataset_R$, respectively, and are memoized at those nodes. It is possible to compute $f_{LR}$ in $O(1)$ time, since it is just the sum of two numbers, $g_{LR}$ and $h_{LR}$. The time it takes to compute $h_{LR}$ depends on the objective-specific heuristic being used. Note that $h_{LR}$ is the sum of a function of a single $\dataset_i$ in Equation \ref{eq:hlr}, i.e., 
\begin{align*}
    h_{LR} &= 
    \sum_{\dataset_\ell \in \textsf{lvs}(\tree_\argmaxsplit(\dataset_L \cup \dataset_R))} \Hfun(\dataset_\ell) \\
    &=
    \sum_{\dataset_\ell \in \textsf{lvs}(\tree_\argmaxsplit(\dataset_L))} \Hfun(\dataset_\ell) + 
    \sum_{\dataset_\ell \in \textsf{lvs}(\tree_\argmaxsplit(\dataset_R))} \Hfun(\dataset_\ell)
\end{align*}
so we compute $\Hfun(\dataset_i)$ once per node, memoize it, and compute $h_{LR}$ in $O(1)$ time given $\Hfun(\dataset_L)$ and  $\Hfun(\dataset_R)$ by summing the two values. 

The time complexity of creating a heap of all tuples is $O(|\children(|\trellisVertex_i|)$, where $|\children(\trellisVertex_i)|$ is the number of children $\trellisVertex_i$ has in trellis $\trellis$. 

Thus it takes $O(|\children(\trellisVertex_i)|) + O(|\Hfun(\dataset_i)|)$ time to create the min heap for $\trellisVertex_i$, where $O(|\Hfun(\dataset_i)|)$ is the time complexity for computing $\Hfun(\dataset_i)$.

Therefore, creating the min heaps for every node in a full trellis takes $\sum_{k=1...n}{\binom{n}{k} O(2^k)} = O(3^n)$ time, when $O(\Hfun(\dataset_i)) = O(2^{|\dataset_i|})$. 
In a sparse trellis, $\sum_{\trellisVertex_i \in \trellis}{O(|\children(\trellisVertex_i)|) + O(|\Hfun(\dataset_i)|)} = \sum_{\trellisVertex_i \in \trellis}{O(|\children(\trellisVertex_i)|)} + \sum_{\trellisVertex_i \in \trellis}{O(|\Hfun(\dataset_i)|)} = \sum_{i = 1...|\trellis|}{(|\trellis| -i)} = O(|\trellis|^2)$, when each trellis vertex, $\trellisVertex_i$, has the maximum possible number of children and  $O(\Hfun(\dataset_i)) = O({|\dataset_i|})$.

(3) When exploring node $\trellisVertex_i$, Algorithm \ref{alg:tree-astar} pops and enqueues an entry from the min heap of every node on the path from $\trellisVertex_i$ to the root in $\tree_\argmaxsplit(\dataset)$, which takes $O(log(|\children(\trellisVertex_k)|)$ for each ancestor $\trellisVertex_k$ in the path.
In a full trellis, the maximum possible path length from the root to $\trellisVertex_i$ is $|\dataset| - |\dataset_i|$.
Thus, in the worst case it takes Algorithm \ref{alg:tree-astar} $\sum_{j=||\dataset_i|...n}{O(log(2^j))}$ time to update the min heaps when $\trellisVertex_i$ is explored.
Therefore, the total amount of time spent updating the min heaps in $\tree_\argmaxsplit(\dataset)$ at line 19 in a full trellis when every vertex is explored is $\sum_{k=1...n}{\binom{n}{k}\sum_{j=k...n}{log(2^k)}} = O(n^2 2^n)$. %$\sum_{k=1...n}{\binom{n}{k}(n-k)log(n-k)} = \frac{1}{2} \frac{(2^{n}-2)}{n} = O(2^n)$. 
In a sparse trellis, the path lengths are maximized when the root, $\trellisVertex_1$, is a parent of all the remaining $|\trellis| -1$ vertices in the trellis, the eldest child of the root, $\trellisVertex_2$, is a parent of all the remaining $|\trellis| -2$ vertices in the trellis, and so on. 
Thus the sum of the total path lengths is bounded from above by $\sum_{i=1...|\trellis|}{\sum_{j=i...|\trellis|}{log(j)}} = O(log(|\trellis|^3))$

Therefore, the time complexity of running Algorithm \ref{alg:tree-astar} on a full trellis is $O(3^n)$, and 
the time complexity of running Algorithm \ref{alg:tree-astar} on a full trellis is $O(|\trellis|^2)$.

\end{proof}

\subsection{Theorem  \ref{THM:OPT_EFFICIENCY} (Optimal Efficiency of A* MAP) Proof}
\label{prf:opt_efficiency}
\begin{proof}
We define a mapping between paths in the trellis representing the state space and the standard space of tree structures via $\tree_\argmaxsplit(\dataset)$ (Eq. \ref{eq:tree_from_trellis}). This mapping is bijective. We have describe a method to find the neighboring states of $\tree_\argmaxsplit(\dataset)$ in the standard space of trees using the trellis in Algorithm \ref{alg:tree-astar} ($\rhd \textsf{Explore new leaves}$). Given the additive nature of the cost functions in the family of costs (Eq. \ref{defn:energy_based_hclustering}), we can maintain the splits/merges in the aforementioned nested min heap in the trellis structure and have parent nodes' $f$ values be computed from their child's min heaps. The result is that Algorithm \ref{alg:tree-astar} exactly follows \astar's best-first according to $f$ with an admissible heuristic in search over the entire space of tree structures.
Therefore the optimal efficiency of Algorithm \ref{alg:tree-astar} follows as a result of the optimal efficiency of the A* algorithm.
\end{proof}

\subsection{Proposition  \ref{THM:HCC_HEURISTIC1} (Admissibility of Hierarchical Correlation Clustering Cost Heuristic) Proof}
\label{prf:hcc-heuristic1}
\begin{proof}
We wish to prove that $\Hfun_{hcc1}(\dataset) <= \argmin_{\tree \in \alltrees(\dataset)}{\EfunT_{hcc}(\tree)}$.
Note that in every hierarchical clustering, every element is eventually separated (at the leaves), thus every edge eventually crosses a cut in every tree, thus $\forall \tree \in \alltrees(\dataset)$, 
\begin{equation}
    \sum_{X_L,X_R \in \textsf{sibs}(\tree)} \quad\;\; \,\,\,\,\,\,\,
         \sum_{\mathclap{\substack{x_i, x_j \in X_L \times X_R}}} w_{ij} \II[w_{ij} > 0] 
        = \sum_{\mathclap{\substack{x_i, x_j \in X}}} w_{ij} \II[w_{ij} > 0]
\end{equation}

Since \[\forall \dataset, \sum_{\mathclap{\substack{x_i,x_j \in X \times X, \\ i < j}}} |w_{ij}| \II[w_{ij} < 0] >= 0\]

We have

\begin{multline*}
      \EfunT(\tree) = \sum_{X_L,X_R \in \textsf{sibs}(\tree)} \EfunS(X_L,X_R) \\
         = \sum_{X_L,X_R \in \textsf{sibs}(\tree)} \quad\;\;\;
         (\sum_{\mathclap{\substack{x_i, x_j \in X_L \times X_R}}} w_{ij} \II[w_{ij} > 0]  + 
        \sum_{\mathclap{\substack{x_i,x_j \in X_L \times X_L, \\ i < j}}} |w_{ij}| \II[w_{ij} < 0] 
        + \sum_{\mathclap{\substack{x_i,x_j \in X_R \times X_R, \\ i < j}}} |w_{ij}| \II[w_{ij} < 0]) \\
         >= \sum_{X_L,X_R \in \textsf{sibs}(\tree)}\quad\;\;\;
         \sum_{\mathclap{\substack{x_i, x_j \in X_L \times X_R}}} w_{ij} \II[w_{ij} > 0] \\
        = \sum_{\mathclap{\substack{x_i, x_j \in X}}} w_{ij} \II[w_{ij} > 0]
        = \Hfun_{hcc1}(\dataset)
\end{multline*}

\end{proof}

\subsection{Heuristic Functions for Ginkgo Jets}\label{app:ginkgoHeuristics}

We want to find heuristic functions to set an upper bound on the likelihood of Ginkgo hierarchies (for more details about Ginkgo see  \cite{ToyJetsShower}). We split the heuristic into internal nodes and leaves.

\subsubsection{ Internal nodes}
We want to find the smallest possible upper bound to 
\bea\label{inner_lh}
 f(t | \lambda, t_{\text{P}})=\frac{1}{1-e^{- \lambda}} \frac{\lambda}{t_{\text{P}}} e^{- \frac{\lambda}{ t_{\text{P}}} t} 
\eea
with $\lambda$ a constant.
We first focus on getting an upper bound for $ t_{\text{P}}$ in the exponential. The upper bound for all the parent masses is given by the squared mass of the root vertex, $t_{root}$ (top of the tree). 

Next, we look for the biggest lower bound on the squared mass $t$ in the exponential of Equation \ref{inner_lh}.
We consider a fully unbalanced tree as the topology with the smallest per level values of $t$ and we bound each level in a bottom up approach (singletons are at the bottom).  For each internal node,  $t$ has to be greater than $t_{cut}$ ($t_{cut}$ is a threshold below which binary splittings stop in Ginkgo). Thus, for each element, we find the smallest parent invariant squared mass  $t_{\text{Pi}}$ that is above the threshold $t_{cut}$, else set this value to $t_{cut}$ and save it to a list named $t_{min}$, that we sort.
Maximizing the number of nodes per level (having a full binary tree) minimizes the values of $t$. We start at the first level, taking the first $N//2$ entries of $t_{min}$ and adding them to our candidate minimum list, $t_{\text{bound}}$  ($N$ is the number of elements to cluster).
For the second level, we take $t_p^0$ as the minimum value in $t_{min}$ that is greater than $t_{cut}$ and bound $t$ by 
$[\tilde{t} (i\%2)+ t_p^0(i//2)]  (j// 2)$ with  $i = 3$ and $j = N \% 2 + N // 2$, where $\tilde{t}$ is the minimum invariant squared mass among all the leaves. The reason is that $(j// 2)$ bounds the total possible number of nodes at current level and $i$ the minimum number of elements of the tree branch below a given vertex at current level.
We calculate this bound level by level, maximizing the possible number of nodes at each level. This is described in Algorithm \ref{alg:t-lower-bound}.
\begin{algorithm}
 \caption{Lower Bound on $t$ in Equation \ref{inner_lh}}
 \label{alg:t-lower-bound}
 \begin{algorithmic}[1]
\Function{LowerBound}{$t_i$, $t_{min}$, $N$}
    \State \textbf{Input: } Elements invariant mass squared: $t_i$, minimum value in $t_{min}$ that is greater than $t_{cut}$: $t_p^0$, list with minimum parent mass squared above $t_{cut}$ for each element: $t_{min}$, and number of elements: $N$.
    \State \textbf{Output: } List that bounds $t$: $t_{bound}$.
    \State \textsf{$\rhd$ Set initial variables}
    \State  $m^2= sort([t_i \text{ for i in N}])$
    \State $t_{bound} = sort(t_{min})[0:N // 2]$
    \State $i = 3$
    \State $j = N \% 2 + N // 2$
    \Do
        \State $t_{bound} += [m^2[0] (i\%2)+ t_p^0(i//2)]  (j// 2))$
        \State $j =  j \% 2 + j // 2 $
        \State $i += 1$
    \doWhile{$len(t_{bound}) < N - 2$}
    \EndFunction
  \end{algorithmic}
\end{algorithm}

Finally, to get a bound for $t_{\text{P}}$ (denoted as $\tilde{t}_P$) in the denominator of Equation \ref{inner_lh} we want the minimum possible values. Here we consider two options for a heuristic:
\begin{itemize}
    \item Admissible heuristic: Thus we take $\tilde{t}_P = t_{bound}+ \tilde{t}$ (given that the parent of any internal node contains at least one more element) except when $t_{bound} < \text{max}\{t_i\}_{i=0}^N$ where we just keep $t_{\text{P}} =t_{bound}$ (see \cite{ToyJetsShower} for more details).
    
    \item h1: $\tilde{t}_P = t_{bound}+ 2 \,\,t_p^0$. Even though we do not have a proof for h1 to be admissible, we gained confidence about it from checking it on a dataset of 5000 Ginkgo jets with up to 9 elements (compared to the exact trellis solution for the MAP tree), as well as the fact that exact \astar with h1 agrees with the exact trellis within 6 significant figures (see Figure \ref{fig:ginkgoCost}).
    
\end{itemize}

Putting it all together, the upper bound for Equation \ref{inner_lh} is
\bea
\sum \log f^{bound} = [-\log(1 - e^{- \lambda}) + \log(\lambda)] \nonumber  - \sum_{i=0}^{N-2} \bigg(\log(\tilde{t}_P ) + \lambda * \frac{t_{bound}^i }{ t_{root}} \bigg)
\eea

\subsubsection{ Leaves}
We want to find the smallest possible upper bound to 
\bea\label{outer_lh}
 f(t | \lambda, t_{\text{P}})=\frac{1}{1-e^{- \lambda}} \bigg(1-e^{- \frac{\lambda}{ t_{\text{P}}} t_{\text{cut}}}\bigg) 
\eea
with $\lambda$ and $t_{cut}$ constants.

We want the maximum possible value of  $t_{\text{P}}$ while still getting an upper bound in  Equation \ref{outer_lh}.
A parent  $t_{\text{P}}$ consists of a pairing of two elements (leaves). We find and sort the same list $t_{min}$ as we did for the internal nodes. 
We do not know the order in which the two children are sampled, and the minimum value happens when each child is sampled last, $t_{\text{Pi}} \rightarrow (\sqrt{t_{\text{Pi}}} - \sqrt{t_i})^2$  (see \cite{ToyJetsShower} for more details). We take the minimum value greater than $t_{cut}$ in $t_{min}$ and refer to it as $t^0_{\text{Pi}}$. Thus we get  $t_{\text{Pi}} =  (\sqrt{t^0_{\text{Pi}}} - \sqrt{t_i})^2$. For the element with greatest invariant squared mass $t_i$ we just keep $t_{\text{Pi}} =t^0_{\text{Pi}}$.
Finally, sum the likelihood over each element $i$.

\subsection{Dasgupta's Cost}
\label{app:dasgupta}

\begin{objective}
\label{ex:dasgupta_cost}
\textbf{\emph{(Dasgupta's Cost)\cite{dasgupta2016cost} }}
Given a graph with vertices of the dataset $\dataset$ and weighted edges representing pairwise similarities between points $\mathcal{W} = \{(i,j,w_{ij}) | i,j \in \{1,..., |\dataset|\} \times \{1,..., |\dataset|\}, i< j, w_{ij} \in \RR^{+}\}$. Dasgupta's cost is defined as:
\begin{equation}
\EfunS(X_i,X_j) = (|X_i| + |X_j|) \sum_{x_i, x_j \in X_i \times X_j} w_{ij} 
\label{eq:dasgupta_potential}
\end{equation}
This is equivalent to the cut-cost definition of Dasgupta's cost with the restriction to binary trees \cite{dasgupta2016cost}. 
\end{objective}

We are given the similarity graph, $\dasguptaGraph(\dataset)$, where the nodes
of the graph refer to the dataset $\dataset$ and edges (denoted $E_{\dasguptaGraph(\dataset)}$), $w_{i,j}$ give similarity between points $x_i, x_j \in \dataset$. Given a subset $\dataset' \subseteq \dataset$, we hope to design a heuristic $h(\cdot)$ such that $h(\dataset') \leq \min_{\tree \in \alltrees(\dataset')} J(\tree)$. 

We can think about Dasgupta's cost as a sum over weighted edges in the graph $E_{\dasguptaGraph(\dataset)}$ with each edge's similarity multiplied by the number of descendant nodes. This leads to a heuristic in which the number of descendants is lower-bounded by having a single descendant:
\begin{align}\label{eq:naive_dasgupta}
    h_\textsf{d1}(\dataset) = \sum_{(i,j) \in E_{\dasguptaGraph(\dataset)}} w_{i,j}
\end{align}

\begin{proposition}{\textbf{\emph{(Admissible Heuristic for Dasgupta)}}}
Equation \ref{eq:naive_dasgupta} is an admissible heuristic for Dasgupta's cost, that is $h_\textsf{d1}(\dataset) \leq J_\textsf{Dasgupta}(\dataset)$.
\end{proposition}

\begin{proof}
We wish to prove that $\Hfun_{dasgupta}(\dataset) <= \argmin_{\tree \in \alltrees(\dataset)}{\EfunT_{dasgupta}(\tree)}$.
\end{proof}
Note that in every hierarchical clustering, every element is eventually separated (at the leaves), thus every edge eventually crosses a cut in every tree, thus $\forall \tree \in \alltrees(\dataset)$, 
\begin{align}
    \sum_{X_L,X_R \in \textsf{sibs}(\tree)} \quad \,\,\,\,\,\,\,\,\,
        %(\sum_{\mathclap{\substack{x_i,x_j \in X_L \times X_L, \\ i < j}}} |w_{ij}| \II[w_{ij} < 0] +
        %\sum_{\mathclap{\substack{x_i,x_j \in X_R \times X_R, \\ i < j}}} |w_{ij}| \II[w_{ij} < 0])
         \sum_{\mathclap{\substack{x_i, x_j \in X_L \times X_R}}} w_{ij} 
        = \,\,\,\,\,\,\,\,\,\sum_{\mathclap{\substack{x_i, x_j \in X}}} w_{ij} 
\end{align}

Since \[\forall \dataset, |\dataset| \in \mathbb{Z}^+ \]

We have

\begin{multline*}
      \EfunT(\tree) = \sum_{X_L,X_R \in \textsf{sibs}(\tree)} \EfunS(X_L,X_R) \\
         = \sum_{X_L,X_R \in \textsf{sibs}(\tree)} \quad\;\;\;
         ((|X_L| + |X_R|) \sum_{x_i, x_j \in X_L \times X_R} w_{ij} ) \\
         >= \sum_{X_L,X_R \in \textsf{sibs}(\tree)}\quad\;\;\;
         \sum_{\mathclap{\substack{x_i, x_j \in X_L \times X_R}}} w_{ij} \\
        = \sum_{\mathclap{\substack{x_i, x_j \in X}}} w_{ij}
        = \Hfun_{Dasgupta}(\dataset)
\end{multline*}

\end{document}